%% file: main.tex
\documentclass[11pt,letterpaper]{article}
\usepackage[T1]{fontenc}
\usepackage[utf8]{inputenc}

\usepackage{amsmath,amssymb,amsthm,mathtools}
\usepackage{newpxtext}                 
\usepackage{newpxmath}                 
\usepackage[scaled=0.90]{FiraSans}     

\usepackage[letterspace=80]{microtype}

\usepackage[letterpaper,left=1.15in,right=1.15in,top=1.1in,bottom=1.15in,headsep=0.32in,footskip=0.45in]{geometry}
\usepackage{xcolor}
\definecolor{ink}{HTML}{1A1F2B}
\definecolor{navy}{HTML}{1F3A63}
\definecolor{refcool}{HTML}{2C5AA0}      
\definecolor{citewarm}{HTML}{B5403B}     
\definecolor{muted}{HTML}{6B6A66}
\definecolor{rulegray}{HTML}{D3D6DC}
\definecolor{thmback}{HTML}{F2F6FC}
\definecolor{lemback}{HTML}{F5F6F8}
\definecolor{lembar}{HTML}{6B7A90}
\definecolor{absback}{HTML}{F7F8FA}

\usepackage{graphicx,booktabs,array,tabularx,multicol}
\usepackage{enumitem}
\setlist{itemsep=1pt,parsep=0pt,topsep=3pt,partopsep=0pt}
\usepackage[font=small,labelfont={sf,bf,color=navy},labelsep=period,skip=6pt]{caption}
\usepackage[round]{natbib}

\usepackage[most]{tcolorbox}
\usepackage{titlesec,titletoc}
\usepackage{fancyhdr}

\titleformat{\section}{\sffamily\Large\bfseries\color{navy}}{\textcolor{refcool}{\thesection}}{0.7em}{}
\titlespacing*{\section}{0pt}{3.2ex plus 1ex minus .2ex}{1.5ex plus .2ex}
\titleformat{\subsection}{\sffamily\large\bfseries\color{navy}}{\textcolor{refcool}{\thesubsection}}{0.6em}{}
\titlespacing*{\subsection}{0pt}{2.4ex plus .8ex minus .2ex}{1.0ex plus .2ex}
\titleformat{\paragraph}[runin]{\sffamily\bfseries\color{navy}}{}{0pt}{}
\titlespacing*{\paragraph}{0pt}{1.4ex plus .5ex minus .2ex}{0.55em}

\titlecontents{section}[2.0em]{\addvspace{0.6pt}\sffamily\fontsize{8.4}{10.2}\selectfont}
  {\contentslabel[\textcolor{refcool}{\bfseries\thecontentslabel}]{2.0em}}{\hspace*{-2.0em}}
  {\titlerule*[0.6pc]{\textcolor{rulegray}{.}}\contentspage}

\fancypagestyle{first}{\fancyhf{}\fancyfoot[C]{\sffamily\footnotesize\color{muted}\thepage}}

\newtheoremstyle{boxed}{0pt}{0pt}{\itshape}{}{\sffamily\bfseries\color{navy}}{.}{0.5em}
  {\thmname{#1}\thmnumber{ #2}\thmnote{\ \textmd{(#3)}}}
\newtheoremstyle{boxedrm}{0pt}{0pt}{\normalfont}{}{\sffamily\bfseries\color{navy}}{.}{0.5em}
  {\thmname{#1}\thmnumber{ #2}\thmnote{\ \textmd{(#3)}}}
\theoremstyle{boxed}
\newtheorem{theorem}{Theorem}[section]
\newtheorem{proposition}[theorem]{Proposition}
\newtheorem{lemma}[theorem]{Lemma}
\newtheorem{corollary}[theorem]{Corollary}
\theoremstyle{boxedrm}

\newtheorem{remark}[theorem]{Remark}
\tcbset{thmbox/.style={enhanced,breakable,frame hidden,sharp corners,boxrule=0pt,
  left=8pt,right=8pt,top=5pt,bottom=5pt,before skip=9pt,after skip=9pt}}
\tcolorboxenvironment{theorem}{thmbox,colback=thmback,borderline west={2.4pt}{0pt}{refcool}}
\tcolorboxenvironment{proposition}{thmbox,colback=thmback,borderline west={2.4pt}{0pt}{refcool}}
\tcolorboxenvironment{corollary}{thmbox,colback=thmback,borderline west={2.4pt}{0pt}{refcool}}
\tcolorboxenvironment{lemma}{thmbox,colback=lemback,borderline west={2.4pt}{0pt}{lembar}}
\tcolorboxenvironment{definition}{thmbox,colback=white,borderline west={2.4pt}{0pt}{citewarm}}
\tcolorboxenvironment{remark}{blanker,breakable,left=8pt,before skip=7pt,after skip=7pt,
  borderline west={1pt}{0pt}{rulegray}}

\makeatletter
\renewenvironment{proof}[1][\proofname]{\par\pushQED{\qed}\normalfont\topsep6\p@\@plus6\p@\relax
  \trivlist\item[\hskip\labelsep{\sffamily\bfseries\color{navy}#1.}]\ignorespaces}{\popQED\endtrivlist\@endpefalse}
\makeatother

\newcommand{\pflab}[2]{{\color{refcool}\ensuremath{\langle#1\rangle}#2}}
\newenvironment{steps}{\begin{list}{}{%
    \setlength{\leftmargin}{2.2em}\setlength{\labelwidth}{2.0em}\setlength{\labelsep}{0.2em}%
    \setlength{\itemsep}{2pt}\setlength{\parsep}{1pt}\setlength{\topsep}{2pt}\setlength{\partopsep}{0pt}%
    }}{\end{list}}
\newcommand{\step}[2]{\item[\pflab{#1}{#2}.]}
\newcommand{\qedstep}[2]{\item[\pflab{#1}{#2}.]{\sffamily\small\color{navy}Q.E.D.}}
\newcommand{\pf}{\par{\sffamily\small\color{muted}Proof:}\ }
\newcommand{\sref}[2]{\pflab{#1}{#2}}
\newcommand{\xsref}[2]{\pflab{#1}{#2}}
\newcommand{\Assume}{{\sffamily\small\color{navy}Assume:}\ }
\newcommand{\Prove}{{\sffamily\small\color{navy}Prove:}\ }
\newcommand{\Let}{{\sffamily\small\color{navy}Let:}\ }
\newcommand{\Suffices}{{\sffamily\small\color{navy}Suffices:}\ }
\newcommand{\Case}{{\sffamily\small\color{navy}Case:}\ }

\newcommand{\markmainend}{}

\newcommand{\R}{\mathbb{R}}
\newcommand{\one}{\mathbf{1}}
\newcommand{\Bn}{\mathcal{B}_n}
\newcommand{\Bno}{\mathcal{B}_n^{\circ}}
\newcommand{\Sn}{\mathfrak{S}_n}
\newcommand{\On}{\mathcal{O}_1(n)}
\newcommand{\Cn}{\mathcal{C}_n}
\newcommand{\Tsp}{\mathcal{T}}
\newcommand{\gauge}{\mathfrak{s}}
\newcommand{\SK}{\mathrm{SK}}
\newcommand{\gradFR}{\operatorname{grad}}
\newcommand{\Pfr}{\Pi}
\newcommand{\dd}{\mathrm{d}}
\newcommand{\eps}{\varepsilon}
\newcommand{\hpre}{h^{\mathrm{pre}}}
\newcommand{\hpost}{h^{\mathrm{post}}}
\newcommand{\ip}[2]{\langle #1,#2\rangle}
\newcommand{\T}{^{\!\top}}
\DeclareMathOperator{\KL}{KL}
\DeclareMathOperator{\diag}{diag}
\DeclareMathOperator{\conv}{conv}
\DeclareMathOperator{\skewop}{skew}
\DeclareMathOperator{\tr}{tr}
\DeclareMathOperator*{\argmin}{arg\,min}

\usepackage[colorlinks=true,linkcolor=refcool,citecolor=citewarm,urlcolor=refcool,
  pdftitle={The Birkhoff Geometry of Manifold-Constrained Hyper-Connections},pdfauthor={Xiaoyu Li, Zhizhou Sha, Chiwun Yang},pdfsubject={Two Channels, Vertex Viscosity, and Sinkhorn as a Retraction},pdfkeywords={deep learning theory, Riemannian optimization, doubly stochastic matrices, hyper-connections, Sinkhorn normalization, mirror descent}]{hyperref}

\begin{document}
\thispagestyle{first}

\noindent{\color{refcool}\rule{\linewidth}{2.2pt}}\par\vspace{7pt}
\noindent{\sffamily\footnotesize\color{muted}\lsstyle PREPRINT\quad\textbullet\quad October 2026}\par\vspace{8pt}
\noindent{\sffamily\bfseries\fontsize{23}{28}\selectfont\color{navy}The Birkhoff Geometry of\\ Manifold-Constrained Hyper-Connections\par}\vspace{7pt}
\noindent{\sffamily\large\color{citewarm}Two Channels, Vertex Viscosity, and Sinkhorn as a Retraction\par}\vspace{12pt}
\noindent\begin{tabular*}{\linewidth}{@{}l@{\extracolsep{\fill}}l@{\extracolsep{\fill}}l@{}}
{\sffamily\bfseries\color{ink}Xiaoyu Li} & {\sffamily\bfseries\color{ink}Zhizhou Sha} & {\sffamily\bfseries\color{ink}Chiwun Yang} \\[1pt]
{\small University of New South Wales} & {\small University of Texas at Austin} & {\small City University of Hong Kong} \\[1pt]
{\sffamily\footnotesize\href{mailto:xiaoyu.li2@unsw.edu.au}{\color{muted}xiaoyu.li2@unsw.edu.au}} & {\sffamily\footnotesize\href{mailto:zhizhousha@utexas.edu}{\color{muted}zhizhousha@utexas.edu}} & {\sffamily\footnotesize\href{mailto:christiannyang37@gmail.com}{\color{muted}christiannyang37@gmail.com}}
\end{tabular*}

\vspace{10pt}

\begin{tcolorbox}[enhanced,colback=absback,frame hidden,sharp corners,boxrule=0pt,
  borderline west={3pt}{0pt}{refcool},left=11pt,right=11pt,top=7pt,bottom=7pt]
{\sffamily\bfseries\footnotesize\color{navy}\lsstyle ABSTRACT}\par\vspace{3pt}
\small Hyper-connections widen the residual stream of a Transformer to $n$ parallel streams. Their manifold-constrained version (mHC) mixes the streams at each layer with a doubly stochastic matrix, which it computes by Sinkhorn normalization of exponentiated logits. We give a geometric theory of this design on the Birkhoff polytope. First, a doubly stochastic mixer splits the stream into a mean channel, on which mHC is exactly a residual network, and a difference channel, which each layer contracts by its second singular value $\sigma_2\le1-n\min_{ij}H_{ij}$. Thus the extra width is a fading memory with a horizon of $1/(1-\sigma_2)$ layers, and among nonnegative mixers only the permutations do not collapse. Second, the Sinkhorn-logit map is a global chart, and its logit gradient is exactly the Fisher--Rao gradient. Thus logit gradient flow follows a squared Fisher--Rao metric, and the straight-through update is exactly entropic mirror descent. Third, under logit gradient flow the logarithm of each entry moves at a rate of at most $4n^3\|\nabla f\|_\infty\,\eps$, where $\eps$ is the distance to the nearest permutation. Thus gradient flow approaches and leaves the vertices only at rate $1/t$, but mirror descent moves at an exponential rate. Fourth, the local convergence factor of Sinkhorn is $\sigma_2^2$, so a fixed iteration budget limits the horizon. Experiments confirm the predicted rates.\par\vspace{4pt}
{\sffamily\footnotesize\color{navy}\textbf{Keywords:}}\ {\footnotesize deep learning theory, Riemannian optimization, doubly stochastic matrices, hyper-connections, Sinkhorn normalization, mirror descent}
\end{tcolorbox}

\vspace{2pt}
\begin{tcolorbox}[enhanced,colback=white,colframe=rulegray,boxrule=0.5pt,arc=2pt,
  left=10pt,right=10pt,top=5pt,bottom=4pt]
{\sffamily\bfseries\footnotesize\color{navy}\lsstyle CONTENTS}\par\vspace{0pt}
\begin{multicols}{2}
{\hypersetup{linkcolor=ink}\makeatletter\@starttoc{toc}\makeatother}
\end{multicols}
\end{tcolorbox}
\suppressfloats[t]
\clearpage

\input{sec_intro}
\input{sec_setup}
\input{sec_signal}
\input{sec_geometry}
\input{sec_viscosity}
\input{sec_sinkhorn}
\input{sec_isometric}
\input{sec_numerics}
\input{sec_conclusion}

\input{back_matter}

\bibliographystyle{plainnat}
\bibliography{refs}

\clearpage
\appendix
\addtocontents{toc}{\protect\addvspace{5pt}\protect\noindent{\protect\sffamily\protect\small\protect\bfseries\protect\color{muted}Appendix}\protect\par}
\noindent{\color{refcool}\rule{\linewidth}{1.6pt}}\par\vspace{8pt}
\noindent{\sffamily\bfseries\LARGE\color{navy}Appendix}\par\vspace{3pt}
\noindent{\sffamily\small\color{muted}The Birkhoff Geometry of Manifold-Constrained Hyper-Connections}\par\vspace{6pt}
\noindent{\color{rulegray}\rule{\linewidth}{0.5pt}}\par\vspace{4pt}
\suppressfloats[t]
\input{app_roadmap}

\input{app_extra}
\input{app_aux}
\input{app_proofs}
\input{app_discussion}

\end{document}

%% file: sec_intro.tex
\section{Introduction}\label{sec:intro}

\suppressfloats[t]\input{fig_teaser}

A residual connection computes $x_{l+1}=x_l+\mathcal F_l(x_l)$. If we unroll this recursion, we get $x_L=x_l+\sum_{i=l}^{L-1}\mathcal F_i(x_i)$. Thus the signal of layer $l$ reaches layer $L$ without change, in the forward pass and in the backward pass. This is the \emph{identity-mapping property} \citep{he2016identity}.

Hyper-connections (HC) \citep{zhu2025hyper} widen the residual stream to $n$ parallel streams. Each layer reads from the streams, writes to them, and mixes them with a learned $n\times n$ matrix. An unconstrained mixer breaks the identity-mapping property: a product of mixers can amplify or attenuate the stream without bound. In runs with 27B parameters, \citet{xie2026mhc} measure composite gains of order $10^3$, together with loss spikes.

Manifold-constrained hyper-connections (mHC) \citep{xie2026mhc} correct this problem. They make each mixer a doubly stochastic matrix,
\begin{equation}\label{eq:mhc-param}
H_l=\SK(\exp W_l),
\end{equation}
where $W_l$ is a matrix of logits and $\SK$ is a fixed number of Sinkhorn--Knopp normalization sweeps \citep{sinkhorn1964relationship}. Doubly stochastic matrices are nonexpansive, and products of doubly stochastic matrices are doubly stochastic. The mHC design trains stably at scale with $n=4$, and DeepSeek-V4 uses it \citep{deepseek2026v4}. Many follow-ups replace the Sinkhorn loop \citep{yang2026mhclite,zhou2026kromhc,lyubinin2026tbpmhc,dandachi2026gomhc,gu2026temper,wang2026accelerating}, measure how trained models use their streams \citep{alimaskina2026stream,peng2026ablate,zhao2026howdoes}, or use orthogonal mixers \citep{sengupta2026jpmhc,liu2026shc,guo2026ohc}. These works are architectural or empirical, and the mHC paper contains no theorems.

Two observations of \citet{xie2026mhc} need an explanation. First, the composite mixers over 30, 60, and 90 layers are almost uniform matrices (their Figure~8). \citet{guo2026ohc} find that a composite of 24 trained mixers has numerical rank one. Second, the projection stops after 20 sweeps, and the composite backward gain drifts up to about $1.6$ (their Figure~7). We explain both observations.

We also ask three general questions. What does the doubly stochastic constraint give? What does it cost? How does gradient descent on the logits move on the polytope?

\paragraph{Contributions.}
One quantity connects all our results: the second singular value $\sigma_2(H)$ of a mixer $H$. It controls the propagation of signals, the speed of learning near the vertices, and the cost of the projection. Figure~\ref{fig:teaser} shows the first two effects.
\begin{enumerate}[leftmargin=1.2em,itemsep=1pt,topsep=2pt]
\item \textbf{Two channels (Section~\ref{sec:signal}).} A doubly stochastic mixer splits the stream into its mean and its differences. On the mean, mHC is exactly a residual network (Theorem~\ref{thm:two-channel}). On the differences, each layer contracts by $\sigma_2(H)\le1-n\min_{ij}H_{ij}$ (Theorem~\ref{thm:contraction}). Thus the extra width is a fading memory with a horizon of $1/(1-\sigma_2)$ layers (Corollary~\ref{cor:fading}). Among nonnegative mixers, only the permutations do not collapse (Proposition~\ref{prop:rigidity}).
\item \textbf{The Sinkhorn-logit chart (Section~\ref{sec:geometry}).} The map $W\mapsto\SK(\exp W)$ is a global chart of the positive doubly stochastic matrices, and its logit gradient is exactly the Fisher--Rao gradient (Lemma~\ref{lem:differential}). Thus logit gradient flow follows a squared Fisher--Rao metric. The straight-through update is exactly entropic mirror descent, and also exactly Riemannian gradient descent with the Sinkhorn retraction (Theorem~\ref{thm:induced}).
\item \textbf{Vertex viscosity (Section~\ref{sec:viscosity}).} Under logit gradient flow, $|\frac{\dd}{\dd t}\log H_{ij}|\le4n^3\|\nabla f\|_\infty\,\eps(H)$, where $\eps(H)$ is the distance to the nearest permutation (Theorem~\ref{thm:viscosity}). Thus the flow approaches and leaves each vertex only at rate $1/t$, for every $C^1$ loss. The Fisher--Rao flow moves at an exponential rate (Theorem~\ref{thm:fr-rate}). This lifts the known gap between softmax and natural policy gradient on the simplex \citep{mei2020global,li2021softmax} to the Birkhoff polytope. The main tool is a sharp pointwise bound for weighted additive least-squares fits (Lemma~\ref{lem:additive}).
\item \textbf{Truncated Sinkhorn (Section~\ref{sec:sinkhorn}).} The local convergence factor of Sinkhorn--Knopp is $\sigma_2^2$ \citep{knight2008sinkhorn}, the same quantity that sets the horizon. Thus a fixed budget of sweeps limits the horizon that mHC can project accurately. We bound the drift of the backward gain that the truncation causes (Proposition~\ref{prop:drift}).
\item \textbf{Isometric mixers (Section~\ref{sec:isometric}).} The collapse-free mixers among the mean-preserving nonexpansive mixers form a group isomorphic to $O(n-1)$. We give its Riemannian geometry, and we state what full orthogonal mixing \citep{guo2026ohc} gives up.
\end{enumerate}
Section~\ref{sec:numerics} confirms the predicted rates numerically. Appendix~\ref{app:proofs} gives structured proofs of all results.

\paragraph{Related work (more in Appendix~\ref{app:related}).}
Earlier work studies rank collapse in attention networks \citep{dong2021attention,noci2022signal}. Our collapse results are deterministic, concern the stream dimension, and use both marginals of the mixer. \citet{douik2019manifold} introduce the Fisher--Rao geometry of positive doubly stochastic matrices and the Sinkhorn retraction. PermELearn \citep{helmbold2009learning} learns permutations by entropic mirror descent with Sinkhorn projections; see also \citet{adams2011ranking,mena2018learning}. Section~\ref{sec:viscosity} lifts two results from the simplex to the polytope. They are the rate gap between softmax and natural policy gradient \citep{mei2020global,li2021softmax} and the replicator equation as a Fisher--Rao flow \citep{shahshahani1979new,hofbauer1998evolutionary}.

%% file: fig_teaser.tex
\begin{figure*}[t]
\centering
\includegraphics[width=\textwidth]{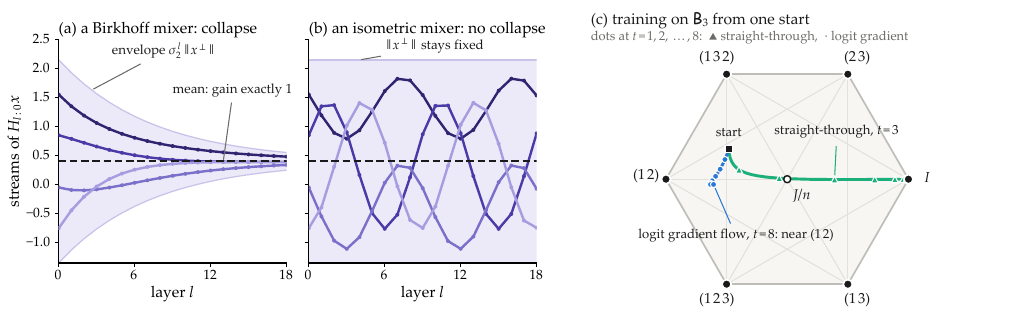}
\caption{The three effects of the theory, all computed. (a) Four streams under one mHC mixer $H=\SK(\exp(4A))$ at every layer. The stream mean (dashed) does not change, and the differences decay inside the envelope $\sigma_2(H)^l\|x^\perp_0\|$ of Theorem~\ref{thm:contraction}. (b) The same start under the isometric mixer $A\in\On$: the differences keep their norm (Proposition~\ref{prop:isometric}). (c) The Birkhoff polytope $\mathcal B_3$ under a linear map that sends the six permutations to a regular hexagon and $J/n$ to its center. The loss $\ip{C}{H}$ has its minimum at $I$. From one start, the straight-through update reaches $I$ by $t=8$ ($\eps=0.003$). Logit gradient flow moves toward the vertex $(1\,2)$ and is near it at $t=8$; it needs $t\approx400$ to reach $\eps=0.001$ at $I$ (Theorems~\ref{thm:viscosity} and~\ref{thm:fr-rate}).}
\label{fig:teaser}
\end{figure*}

%% file: sec_setup.tex
\section{Setting and notation}\label{sec:setup}

\paragraph{Hyper-connections.}
We follow \citet{xie2026mhc}. The input of layer $l$ is a stream $x_l\in\R^{n\times d}$, and its $n$ rows are the streams. A layer with block $\mathcal F_l:\R^{1\times d}\to\R^{1\times d}$ (attention or feed-forward) computes
\begin{equation}\label{eq:hc}
x_{l+1}=H_l\,x_l+(\hpost_l)\T\,\mathcal F_l(\hpre_l\,x_l).
\end{equation}
Here $H_l\in\R^{n\times n}$ is the \emph{mixer} (the matrix $\mathcal H^{\mathrm{res}}_l$ of \citet{xie2026mhc}), $\hpre_l\in\R^{1\times n}$ is the read-out, and $\hpost_l\in\R^{1\times n}$ is the write-in. The \emph{composite mixer} is $H_{L:l}:=H_{L-1}\cdots H_l$, with $H_{l:l}:=I$.

In mHC, $\hpre_l\in(0,1)^{1\times n}$ and $\hpost_l\in(0,2)^{1\times n}$ are outputs of sigmoids, and $H_l=\SK(\exp W_l)$. The logits $W_l$ are the sum of a static bias and an input-dependent term. The map $\SK$ alternately normalizes the columns and the rows, $t_{\max}=20$ times, and it stops after a row normalization. Thus $H_l\one=\one$ holds exactly, and the column sums are approximately one. Our theory treats the logits as the variables; Remark~\ref{rem:dynamic} relates them to the trainable parameters. The term \emph{exact projection} means $t_{\max}=\infty$, and Section~\ref{sec:sinkhorn} studies the truncation.

\paragraph{Doubly stochastic matrices.}
Let $n\ge2$. Let $\one\in\R^n$ be the all-ones vector, let $J:=\one\one\T$, and let $Q:=I-J/n$ be the orthogonal projector onto $\one^\perp$. The \emph{Birkhoff polytope} is
\[
\Bn:=\{H\in\R^{n\times n}:\ H\ge0,\ H\one=\one,\ H\T\one=\one\}.
\]
Its relative interior $\Bno$ is the set of positive doubly stochastic matrices. We identify a permutation $\pi$ of $[n]$ with its matrix $P_\pi$, where $(P_\pi)_{ij}=1$ if $j=\pi(i)$ and $0$ otherwise. We write $\Sn$ for the set of permutation matrices. By the Birkhoff--von Neumann theorem, $\Sn$ is the set of vertices of $\Bn$ \citep{birkhoff1946tres}. We use three facts many times.
\begin{enumerate}[leftmargin=2.2em,itemsep=0pt,topsep=1pt]
\item[(F1)] Every $H\in\Bn$ satisfies $\|H\|_2\le1$.
\item[(F2)] Products of matrices in $\Bn$ (in $\Bno$) are in $\Bn$ (in $\Bno$).
\item[(F3)] Every $H\in\Bn$ satisfies $HJ=JH=J$. Thus $H$ commutes with $Q$, and $H$ maps $\one^\perp$ into $\one^\perp$.
\end{enumerate}
For $H\in\Bn$ we write $\delta(H):=\min_{ij}H_{ij}$ and $\sigma_2(H):=\|HQ\|_2$. By (F3), $\sigma_2(H)$ is the second-largest singular value of $H$. The \emph{vertex proximity} of $H$ is
\begin{equation}\label{eq:vertexeps}
\eps_\pi(H):=\max_{i,\ j\ne\pi(i)}H_{ij},\qquad \eps(H):=\min_{\pi\in\Sn}\eps_\pi(H).
\end{equation}
It is the largest entry off the support of the nearest permutation. Thus $\eps(H)=0$ if and only if $H\in\Sn$. Also, $|H_{ij}-(P_\pi)_{ij}|\le(n-1)\eps_\pi(H)$ for all $(i,j)$.

\paragraph{Fisher--Rao geometry.}
Let $\Tsp:=\{\xi\in\R^{n\times n}:\ \xi\one=0,\ \xi\T\one=0\}$ and $\gauge:=\{a\one\T+\one b\T:\ a,b\in\R^n\}$. For the Frobenius inner product $\ip{\cdot}{\cdot}$, we have $\gauge=\Tsp^\perp$, $\dim\Tsp=(n-1)^2$, and $\dim\gauge=2n-1$. The set $\Bno$ is an open subset of the affine space $J/n+\Tsp$. Thus $\Bno$ is a manifold, and its tangent space is $\Tsp$ at every point. As \citet{douik2019manifold} do, we give $\Bno$ the \emph{Fisher--Rao metric}
\begin{equation}\label{eq:FR}
g_H(\xi,\zeta):=\sum_{ij}\frac{\xi_{ij}\zeta_{ij}}{H_{ij}}=\ip{\xi\oslash H}{\zeta},
\end{equation}
where $\odot$ and $\oslash$ are the entrywise product and quotient. The $g_H$-orthogonal complement of $\Tsp$ in $\R^{n\times n}$ is $H\odot\gauge$ (Lemma~\ref{lem:gauge-complement}). Let $\Pfr_H$ be the $g_H$-orthogonal projection onto $\Tsp$. Then $\Pfr_H(\xi)=\xi-H\odot(a\one\T+\one b\T)$, where $(a,b)$ solves
\begin{equation}\label{eq:FRproj}
a_i+\sum_jH_{ij}b_j=\sum_j\xi_{ij},\qquad b_j+\sum_iH_{ij}a_i=\sum_i\xi_{ij}.
\end{equation}
Equivalently, $a_i+b_j$ is the least-squares fit of $\xi\oslash H$ by an additive model, with weights $H_{ij}$ (Lemma~\ref{lem:fit-normal}). Section~\ref{sec:viscosity} depends on this fact. The Fisher--Rao gradient of a differentiable $f:\R^{n\times n}\to\R$ is \citep{douik2019manifold}
\begin{equation}\label{eq:FRgrad}
\gradFR f(H)=\Pfr_H\big(H\odot\nabla f(H)\big).
\end{equation}
We write $\|X\|_\infty:=\max_{ij}|X_{ij}|$ and $\KL(H\,\|\,K):=\sum_{ij}(H_{ij}\log(H_{ij}/K_{ij})-H_{ij}+K_{ij})$. Table~\ref{tab:notation} in Appendix~\ref{app:related} lists the notation.

%% file: sec_signal.tex
\section{The mean channel and the difference channel}\label{sec:signal}

Every $x\in\R^{n\times d}$ has a unique decomposition
\begin{equation}\label{eq:channels}
x=\one\bar x\T+x^{\perp},\qquad \bar x:=\tfrac1n x\T\one,\quad x^{\perp}:=Qx.
\end{equation}
We call $\bar x\in\R^d$ the \emph{stream mean} and $x^\perp$ the \emph{stream differences}; the rows of $x^\perp$ sum to zero. The first theorem shows that a doubly stochastic mixer does not couple the two parts. On the mean, mHC is a residual network.

\begin{theorem}[Two channels]\label{thm:two-channel}
Let $H_l\in\Bn$ for all $l$. Let $u_l:=\mathcal F_l(\hpre_l x_l)$ be the block output, and let $\bar h_l:=\frac1n\hpost_l\one$. Then \eqref{eq:hc} is equivalent to the two recursions
\begin{equation}\label{eq:two-channel}
\bar x_{l+1}\T=\bar x_l\T+\bar h_l\,u_l,\qquad
x^{\perp}_{l+1}=H_l\,x^{\perp}_l+Q(\hpost_l)\T u_l .
\end{equation}
The Jacobian of the residual path $x_l\mapsto x_L$, with all block outputs held fixed, is the map $x\mapsto H_{L:l}\,x$, and its adjoint is $y\mapsto H_{L:l}\T y$. Both maps fix $\one$. Thus the gain of the mean channel is exactly one in both directions, for all mixers in $\Bn$.
\end{theorem}

The mean recursion is the residual network of \citet{he2016identity}, with block outputs scaled by $\bar h_l\in(0,2)$. This is the exact content of the identity-mapping property that mHC restores. All other effects of the mixers stay in the difference channel. \citet{guo2026ohc} observe the decomposition and the contraction of one layer. Our new content is the rate.

\begin{theorem}[Contraction]\label{thm:contraction}
\begin{enumerate}[leftmargin=1.4em,itemsep=0pt,topsep=1pt,label=(\alph*)]
\item For every $H\in\Bn$, $\sigma_2(H)\le 1-n\,\delta(H)$. Equality holds for $H=(1-t)P_\pi+tJ/n$, $t\in[0,1]$.
\item For $l<L$, $\sigma_2(H_{L:l})\le\prod_{i=l}^{L-1}\sigma_2(H_i)\le\exp\big(-n\sum_{i=l}^{L-1}\delta(H_i)\big)$.
\item $\max_{jk}\big|(H_{L:l})_{jk}-\tfrac1n\big|\le\sigma_2(H_{L:l})$. Thus $H_{L:l}\to J/n$ entrywise if $\sum_i\delta(H_i)\to\infty$.
\item Let $H=\SK(\exp W)$ with $\|W\|_\infty\le M$, for the exact projection or for any number of sweeps that ends with a row normalization. Then $\delta(H)\ge e^{-4M}/n$. For the exact projection, $\sigma_2(H)\le1-e^{-4M}$.
\end{enumerate}
\end{theorem}

\begin{proof}[Proof sketch]
Write $H=n\delta\,(J/n)+(1-n\delta)R$ with $R\in\Bn$; this is a Doeblin splitting. On $\one^\perp$ the uniform part is zero, and $R$ is nonexpansive by (F1). This gives (a). Products act on $\one^\perp$ factor by factor, which gives (b). For (d), the column scalings have ratios at most $e^{2M}$. Thus two entries in one row differ by a factor of at most $e^{4M}$. Also, each row has an entry of at least $1/n$.
\end{proof}

\begin{corollary}[Width as a fading memory]\label{cor:fading}
Let $\rho_i:=\sigma_2(H_i)$ and $\rho_{L:l}:=\prod_{i=l}^{L-1}\rho_i$. Then for $l<L$,
\begin{equation}\label{eq:fading}
\|x^{\perp}_L\|_F\le\rho_{L:l}\|x^{\perp}_l\|_F+\sum_{i=l}^{L-1}\rho_{L:i+1}\|Q(\hpost_i)\T\|_2\|u_i\|_2 .
\end{equation}
For the backward map, $\|QH_{L:l}\T y\|_F\le\rho_{L:l}\|Qy\|_F$. If $\rho_i\le\rho<1$ for all $i$, then $\rho^{k}\le e^{-k/\tau}$ with the \emph{horizon}
\begin{equation}\label{eq:horizon}
\tau:=1/(1-\rho).
\end{equation}
\end{corollary}

Thus an mHC network is a residual network on the stream mean, plus an $(n-1)$-dimensional memory of the last $O(\tau)$ block outputs.

\paragraph{What this explains.}
In Figure~8 of \citet{xie2026mhc}, the trained mixers of single layers have many entries between $0.05$ and $0.5$. But the composites over 30, 60, and 90 layers have all entries between about $0.17$ and $0.35$. Theorem~\ref{thm:contraction} shows that this result is necessary. With $n=4$ and a minimum entry of $\delta=0.02$ per layer, $\sigma_2(H_{30:0})\le e^{-2.4}<0.091$, so each entry of the composite is within $0.091$ of $1/n$. Thus the early layers reach the late layers only through the stream mean.

\paragraph{Collapse-free mixers.}
HC arranges blocks in sequence or in parallel with $0/1$ mixers, that is, with vertices of $\Bn$ \citep{zhu2025hyper}. An interior mixer is a convex combination of such arrangements, and by Theorem~\ref{thm:contraction} their composites decay to the uniform average. We call a mixer $A$ \emph{collapse-free} if $\|Ax\|_2=\|x\|_2$ for all $x\perp\one$. The next result shows that, among nonnegative mixers, only the permutations are collapse-free. It also shows that the signed collapse-free mixers form an orthogonal group.

\begin{proposition}[Rigidity]\label{prop:rigidity}
Let $\Cn:=\{A\in\R^{n\times n}:\ A\one=\one,\ \|A\|_2\le1\}$ and $\On:=\{A\in\Cn:\ A\T A=I\}$.
\begin{enumerate}[leftmargin=1.4em,itemsep=0pt,topsep=1pt,label=(\alph*)]
\item Every $A\in\Cn$ satisfies $A\T\one=\one$, and $A$ is collapse-free if and only if $A\in\On$. Also, $\Cn=\conv\On$ and $\Bn=\Cn\cap\R^{n\times n}_{\ge0}$.
\item $\Bn\cap\On=\Sn$.
\item For $H\in\Bn$, $\sigma_2(H)=1$ if and only if $H$ is \emph{partly decomposable}. That is, there are sets $\mathcal R,\mathcal S\subset[n]$ with $0<|\mathcal R|=|\mathcal S|<n$ such that $H_{ij}=0$ if exactly one of $i\in\mathcal R$ and $j\in\mathcal S$ holds. In particular, $\sigma_2(H)<1$ for every $H\in\Bno$.
\item Let $H\in\Bn$ satisfy $\sigma_n(H)\ge1-\vartheta$ with $4n\vartheta<1$. Then $\eps_\pi(H)\le4n\vartheta$ for some $\pi\in\Sn$.
\end{enumerate}
\end{proposition}

\citet{guo2026ohc} also note part (b). By part (a), the collapse-free mixers in $\Cn$ form the group $\On\cong O(n-1)$ (Section~\ref{sec:isometric}). The proposition gives a square of design spaces:
\[
\begin{array}{ccc}
\Sn & \subset & \On\cong O(n-1)\\[-1pt]
\cap & & \cap\\[-1pt]
\Bn=\conv\Sn & \subset & \Cn=\conv\On .
\end{array}
\]
The top row contains isometries (no collapse and no explosion), and the bottom row contains contractions (no explosion). Nonnegativity, the design principle of mHC, selects the left column, and unit gain in every direction selects the top row. The two meet only at the permutations. Thus, in the design space of mHC, ``collapse-free'' means ``at a vertex''. Part (d) gives a quantitative version: a mixer that almost preserves all stream differences is near a vertex. Sections~\ref{sec:geometry} and~\ref{sec:viscosity} show that the Sinkhorn-logit parameterization approaches the vertices only slowly.

%% file: sec_geometry.tex
\section{The Sinkhorn-logit chart}\label{sec:geometry}

Let $\Phi(W):=\SK(\exp W)$ denote the exact projection. By Sinkhorn's theorem \citep{sinkhorn1964relationship}, every positive matrix $K$ has a unique doubly stochastic scaling $D_1KD_2$, with $D_1,D_2$ positive and diagonal. Thus $\Phi(W)=\diag(e^{\alpha})\exp(W)\diag(e^{\beta})$ for some $\alpha,\beta\in\R^n$. These vectors are unique, up to a constant that we add to $\alpha$ and subtract from $\beta$.

\begin{lemma}[Global chart]\label{lem:chart}
The map $\Phi:\R^{n\times n}\to\Bno$ is smooth and surjective, and $\Phi(W)=\Phi(W')$ if and only if $W-W'\in\gauge$. The restriction of $\Phi$ to $\Tsp$ is a diffeomorphism onto $\Bno$, with inverse $H\mapsto Q(\log H)Q$. If $W_k\in\Tsp$ and $\Phi(W_k)\to P_\pi$, then $\|W_k\|_\infty\to\infty$.
\end{lemma}

The chart is natural for matrix scaling. The set $\Bno$ is the quotient of the positive matrices by the scalings $K\mapsto D_1KD_2$, and $W\mapsto W+\gauge$ is the same action in logarithmic coordinates. The next lemma identifies the gradient in logit space.

\begin{lemma}[Differential of the chart]\label{lem:differential}
Let $H=\Phi(W)$, and define $A_H:\R^{n\times n}\to\Tsp$ by $A_H(\xi):=\Pfr_H(H\odot\xi)$.
\begin{enumerate}[leftmargin=1.4em,itemsep=0pt,topsep=1pt,label=(\alph*)]
\item $\dd\Phi_W(V)=A_H(V)$ for every $V\in\R^{n\times n}$.
\item $A_H$ is self-adjoint and positive semidefinite, with kernel $\gauge$. On $\Tsp$ it is the inverse of the Fisher--Rao metric tensor: $g_H(\xi,\zeta)=\ip{A_H^{-1}\xi}{\zeta}$ for $\xi,\zeta\in\Tsp$.
\item For every differentiable $f$,
\begin{equation}\label{eq:logit-grad}
\nabla_W(f\circ\Phi)(W)=A_H\big(\nabla f(H)\big)=\gradFR f(H).
\end{equation}
\end{enumerate}
\end{lemma}

Part (c) says that backpropagation through Sinkhorn computes exactly the Fisher--Rao gradient of \citet{douik2019manifold}. But the training step applies this matrix to $W$, not to $H$. Thus the step applies the inverse metric twice. The next theorem makes this statement precise. We write $H^{(k)}$ for the $k$-th training iterate, to separate it from the mixer $H_l$ of layer $l$.

\begin{theorem}[Induced dynamics]\label{thm:induced}
Let $f$ be $C^1$, and let $H(t)=\Phi(W(t))$.
\begin{enumerate}[leftmargin=1.4em,itemsep=1pt,topsep=1pt,label=(\alph*)]
\item \textup{(Logit gradient flow.)} If $\dot W=-\nabla_W(f\circ\Phi)(W)$, then $\dot H=-A_H^{2}\,\nabla f(H)$. This flow is the Riemannian gradient flow of $f$ for the \emph{squared Fisher--Rao metric} $\mathcal G_H(\xi,\zeta):=\ip{A_H^{-1}\xi}{A_H^{-1}\zeta}$ on $\Tsp$.
\item \textup{(Straight-through flow.)} If $\dot W=-\nabla f(\Phi(W))$, then $\dot H=-\gradFR f(H)$. This flow is the Fisher--Rao gradient flow, and it is also the entropic mirror flow $\frac{\dd}{\dd t}\log H\in-\nabla f(H)+\gauge$.
\item \textup{(Three-way identity.)} Let $\eta>0$, $H^{(k)}=\Phi(W^{(k)})$, $\Gamma^{(k)}:=\nabla f(H^{(k)})$, and $W^{(k+1)}=W^{(k)}-\eta\Gamma^{(k)}$. Then
\begin{align}
H^{(k+1)}&=\SK\big(H^{(k)}\odot e^{-\eta\Gamma^{(k)}}\big)\notag\\
&=\argmin_{H\in\Bn}\big\{\eta\ip{\Gamma^{(k)}}{H}+\KL(H\,\|\,H^{(k)})\big\}\notag\\
&=R_{H^{(k)}}\big(-\eta\,\gradFR f(H^{(k)})\big),\label{eq:three-way}
\end{align}
where $R_H(\xi):=\SK(H\odot\exp(\xi\oslash H))$ is the Sinkhorn retraction of \citet{douik2019manifold}.
\end{enumerate}
\end{theorem}

Thus one update is three methods. The first expression in \eqref{eq:three-way} is an exponentiated-gradient step \citep{helmbold2009learning}, and the second is an entropic mirror descent step \citep{beck2003mirror}. The third is a Riemannian gradient descent step on $(\Bno,g)$ with a first-order retraction.

\begin{proof}[Proof sketch]
Parts (a) and (b) follow from the chain rule and Lemma~\ref{lem:differential}. Part (c) depends on one fact: the Sinkhorn projection does not change under diagonal scalings, $\SK(D_1KD_2)=\SK(K)$. Each of the three updates produces $H^{(k)}\odot e^{-\eta\Gamma^{(k)}}$ up to such a scaling.
\end{proof}

\begin{remark}[The simplex as a check]\label{rem:simplex}
For one row, that is, for the simplex with $p=\mathrm{softmax}(w)$, the analogue of $A_H$ is $A_p=\diag(p)-pp\T$, the inverse Fisher--Rao tensor. The straight-through flow $\dot w=-\nabla f(p)$ gives $\dot p=-A_p\nabla f(p)$. This is the replicator equation, which is the Fisher--Rao gradient flow \citep{shahshahani1979new,hofbauer1998evolutionary}. The exact logit flow gives $\dot p=-A_p^{2}\nabla f(p)$. Theorem~\ref{thm:induced} is the doubly stochastic version of this pair of flows. The gauge $\gauge$ takes the place of the single normalizing constant.
\end{remark}

\begin{remark}[Cost]\label{rem:cost}
mHC backpropagates through the Sinkhorn loop with a dedicated kernel \citep[Sec.~4.3]{xie2026mhc}. By Theorem~\ref{thm:induced}(c), the straight-through update skips this backward pass, and it is not an approximation. It is exact entropic mirror descent, a natural-gradient method with its own convergence theory \citep{beck2003mirror}. Remark~\ref{rem:truncated-three-way} discusses the truncated projection.
\end{remark}

\begin{remark}[Dynamic logits]\label{rem:dynamic}
In mHC the logits are $W=\lambda\,\mathrm{mat}(\tilde x\varphi)+b$. Here $\lambda\in\R$, $\varphi$, and $b$ are trainable, and $\tilde x\in\R^{1\times nd}$ is a normalized vector that depends on the input. For one input, gradient flow on $b$ induces exactly $\dot W=-\nabla_W(f\circ\Phi)$. Gradient flow on $\varphi$ induces $\dot W=-\lambda^2\|\tilde x\|_2^2\,\nabla_W(f\circ\Phi)$, the same flow at a different speed. For a batch, the update averages the logit gradients of the inputs with the Gram weights $\lambda^2\ip{\tilde x}{\tilde x'}$. Thus every statement about the logit flow holds input by input, up to this positive semidefinite weighting.
\end{remark}

%% file: sec_viscosity.tex
\section{Vertex viscosity}\label{sec:viscosity}

By Proposition~\ref{prop:rigidity} a nonnegative mixer is collapse-free only at a vertex, and by Lemma~\ref{lem:chart} the vertices are at infinity in the logit chart. How fast can training move toward a vertex, or away from it? We compare the two flows of Theorem~\ref{thm:induced}: the logit gradient flow moves at a polynomial rate, and the Fisher--Rao flow moves at an exponential rate. On the simplex, earlier work shows this gap: softmax policy gradient converges at rate $O(1/t)$, and natural policy gradient converges linearly \citep{mei2020global,li2021softmax}. We lift the gap to $\Bn$, for every $C^1$ loss $f$ and in both directions. Below, $G$ is a bound on $\|\nabla f(H(t))\|_\infty$ along the trajectory.

The mechanism is visible in \eqref{eq:FRgrad}. By Lemma~\ref{lem:fit-normal}, $\gradFR f(H)=H\odot(\nabla f(H)-m)$, where $m=a\one\T+\one b\T$ is the weighted additive fit of $\nabla f(H)$. If $m$ is bounded, then each entry of the gradient is proportional to the corresponding entry of $H$. On the simplex the fit is one number, the $p$-weighted mean of the gradient, and the bound is immediate. On $\Bn$ the fit has $2n-1$ parameters, and its weights can be very unbalanced. The next lemma gives the bound that we need.

\begin{lemma}[Weighted additive fits]\label{lem:additive}
Let $\omega\in\R^{n\times n}_{>0}$ and $C\in\R^{n\times n}$. The problem
\[
\min_{a,b\in\R^n}\ \sum_{i,j}\omega_{ij}\,(C_{ij}-a_i-b_j)^2
\]
has a minimizer, and its fitted values $m_{ij}:=a_i+b_j$ are unique. They satisfy
\begin{equation}\label{eq:additive}
\|m\|_\infty\le(2n-1)\,\|C\|_\infty .
\end{equation}
The constant $2n-1$ is sharp.
\end{lemma}

\begin{proof}[Proof sketch]
By the Cauchy--Binet formula and Cramer's rule, a weighted least-squares solution is a convex combination of the solutions of the square subsystems \citep{bental1990geometric}. For the additive model, the nonsingular square subsystems are the spanning trees of the complete bipartite graph $K_{n,n}$. On a tree, the fitted value of a cell is an alternating sum of data along a path with at most $2n-1$ edges. A Hamiltonian path with data $\pm1$ shows that the bound is sharp.
\end{proof}

\begin{lemma}[Small gradients near vertices]\label{lem:small-grad}
Let $H\in\Bno$ and $G:=\|\nabla f(H)\|_\infty$. Then $|\gradFR f(H)_{ij}|\le 2nG\,H_{ij}$ for all $(i,j)$, and $\|\gradFR f(H)\|_\infty\le 2n(n-1)\,G\,\eps(H)$. Consequently, for every $W$ and $H=\Phi(W)$,
\begin{equation}\label{eq:flat}
\|\nabla_W(f\circ\Phi)(W)\|_\infty\le 2n(n-1)\,\|\nabla f(H)\|_\infty\,\eps(H).
\end{equation}
\end{lemma}

The bound \eqref{eq:flat} is a property of the landscape, not of a dynamics. If $\Phi(W_k)\to P_\pi$, then the logit gradient of every loss goes to zero. Thus every vertex is an asymptotic critical point of $f\circ\Phi$, also if $P_\pi$ does not minimize $f$.

\begin{theorem}[Vertex viscosity of the logit flow]\label{thm:viscosity}
Let $W(t)$, $t\in[0,T)$, solve $\dot W=-\nabla_W(f\circ\Phi)(W)$, let $H(t)=\Phi(W(t))$, and let $G:=\sup_{t<T}\|\nabla f(H(t))\|_\infty<\infty$. Then for all $i,j$ and $t$,
\begin{equation}\label{eq:viscosity}
\Big|\frac{\dd}{\dd t}\log H_{ij}(t)\Big|\ \le\ 4n^{3}\,G\,\eps(H(t)).
\end{equation}
Let $\eps_0:=\eps(H(0))$ and $c:=4n^3G$. Then
\begin{align}
\eps(H(t))&\ge\frac{\eps_0}{1+c\,\eps_0\,t}\quad\text{for } t<T,\label{eq:approach}\\
\eps(H(t))&\le\frac{\eps_0}{1-c\,\eps_0\,t}\quad\text{for } t<T,\ t<\tfrac{1}{c\eps_0}.\label{eq:escape}
\end{align}
Thus the flow needs time at least $(1/\eps-1/\eps_0)/c$ to reach a proximity $\eps<\eps_0$, and time at least $1/(2c\,\eps_0)$ to double the proximity.
\end{theorem}

\input{fig_main}

\begin{proof}[Proof sketch]
Along the flow, $\log H_{ij}=W_{ij}+\alpha_i+\beta_j$. The logit velocity $\dot W$ is the Fisher--Rao gradient, so it is $O(\eps)$ by Lemma~\ref{lem:small-grad}. If we differentiate the marginal constraints, we see that $\dot\alpha_i+\dot\beta_j$ is the weighted additive fit of $-\dot W$. Thus it is also $O(\eps)$ by Lemma~\ref{lem:additive}. The result is $|\dot\eps|\le c\,\eps^2$, which we integrate with Dini derivatives because $\eps$ is only Lipschitz.
\end{proof}

\begin{theorem}[Exponential rate of the Fisher--Rao flow]\label{thm:fr-rate}
Let $H(t)$ solve $\dot H=-\gradFR f(H)$, and let $G$ be as above. Then $|\frac{\dd}{\dd t}\log H_{ij}(t)|\le 2nG$, and therefore
\begin{equation}\label{eq:fr-bounds}
\eps_0\,e^{-2nGt}\ \le\ \eps(H(t))\ \le\ \eps_0\,e^{2nGt}.
\end{equation}
The loss $f_\pi(H):=-\ip{P_\pi}{H}$ (thus $G=1$) attains the exponential rate: for every start $H(0)=\Phi(W_0)$,
\begin{equation}\label{eq:fr-attained}
\eps(H(t))\le(n-1)^{n-1}\,e^{2\|W_0\|_\infty}\,e^{-t}\quad\text{for all } t\ge0.
\end{equation}
For the same loss and start, the logit gradient flow satisfies \eqref{eq:approach} with $c=4n^3$.
\end{theorem}

\begin{proof}[Proof sketch]
Along the flow, $H(t)$ is the scaling of $\exp(W_0)\odot(P_\pi+e^{-t}(J-P_\pi))$. Thus each entry off $\pi$ has a factor $e^{-t}$, and the scaling vectors could, in principle, cancel it. They cannot. A chain of entries off $\pi$, none of them small, closes into a cycle. The product of the entries of a cycle does not depend on the scaling, up to bounded factors.
\end{proof}

For a linear loss with a unique optimal permutation, the Fisher--Rao trajectory is the entropic trajectory of the assignment linear program. This trajectory converges exponentially \citep{cominetti1994asymptotic}. Theorem~\ref{thm:fr-rate} is a self-contained case with explicit constants. For $n=2$ both flows have closed forms, and the comparison becomes a pair of scalar equations.

\begin{proposition}[Two streams]\label{prop:n2}
Let $n=2$ and $w:=W_{11}+W_{22}-W_{12}-W_{21}$. Then $\Phi(W)=\begin{psmallmatrix}p&q\\q&p\end{psmallmatrix}$ with $q=1-p$ and $\log\frac pq=\frac w2$, and $q=\eps(\Phi(W))$ if $p\ge\frac12$. For $f(H)=\ip{C}{H}$, let $\Delta:=C_{12}+C_{21}-C_{11}-C_{22}$; thus $\Delta>0$ means that $I$ is the optimal vertex. Then
\begin{equation}\label{eq:n2}
\begin{aligned}
\text{Fisher--Rao:}\ \ &\dot w=\Delta, & \dot q&=-\tfrac{\Delta}{2}\,q(1-q),\\
\text{logit:}\ \ &\dot w=2\Delta pq, & \dot q&=-\Delta\,q^{2}(1-q)^{2}.
\end{aligned}
\end{equation}
For $\Delta<0$, the logit flow satisfies $q(t)\le q_0/(1-|\Delta|q_0t)$.
\end{proposition}

Near the vertex ($q\to0$) the Fisher--Rao flow is exponential, $\dot q\approx-\frac{\Delta}{2}q$, and the logit flow is quadratic, $\dot q\approx-\Delta q^{2}$, so $q(t)\approx1/(\Delta t)$. For $\Delta<0$ the same equations describe the escape. The Fisher--Rao escape is logistic, and the logit flow needs time at least $1/(2|\Delta|q_0)$ to double $q$. Thus the exponent in Theorem~\ref{thm:viscosity} is sharp. We do not claim that the constant $4n^3$ is sharp (Remark~\ref{rem:constant}). The next result converts Theorem~\ref{thm:viscosity} into a time cost of a long horizon.

\begin{corollary}[The price of a long horizon]\label{cor:price}
Let $c_0\in(0,1)$. Assume that the residual path of an $L$-layer network preserves the difference channel up to the factor $c_0$, that is, $\sigma_n(H_{L:0})\ge c_0$. Let $\eps_\star:=8n(n-1)\log(1/c_0)/L<1$. Then at least half of the layers satisfy $\eps(H_l)\le\eps_\star$. From a common start $\eps_0>\eps_\star$, the logit gradient flow needs time at least $(1/\eps_\star-1/\eps_0)/(4n^3G)$ to reach this configuration, which is linear in $L$. The corresponding lower bound for the Fisher--Rao flow, $\log(\eps_0/\eps_\star)/(2nG)$, grows only as $\log L$.
\end{corollary}

Thus a long memory is expensive for the logit parameterization in two ways. By Proposition~\ref{prop:rigidity}(d), at least half of the layers must be near a vertex. By Theorem~\ref{thm:viscosity}, gradient flow reaches such layers only after a time proportional to the depth.

\begin{remark}[Adaptive optimizers]\label{rem:adam}
mHC trains with AdamW \citep{kingma2015adam}, which divides each update by a running root-mean-square of the gradient coordinate. Near a vertex the logit gradient is $O(\eps)$ by \eqref{eq:flat}, but its running root-mean-square is also $O(\eps)$. Thus the normalized update is $O(\eta)$, and Adam moves at an exponential rate (Figure~\ref{fig:main}(b,c)). Theorem~\ref{thm:viscosity} is therefore a statement about gradient flow, gradient descent, and the landscape. The landscape is flat at every vertex for every loss, so every optimizer receives a gradient signal of size $O(\eps)$. The collapse of Section~\ref{sec:signal} and the budget limit of Section~\ref{sec:sinkhorn} do not depend on the optimizer.
\end{remark}

%% file: fig_main.tex
\begin{figure*}[t]
\centering
\includegraphics[width=\textwidth]{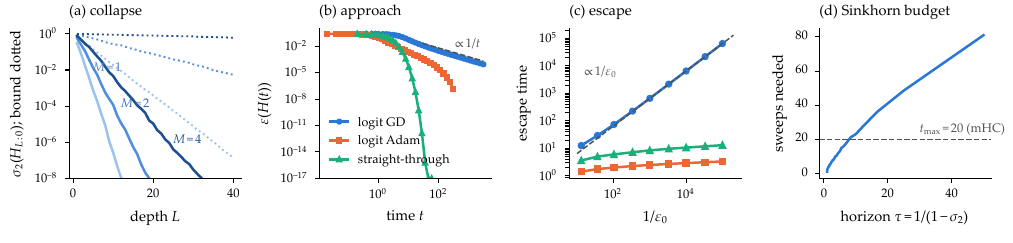}
\caption{Numerical checks with $n=4$. (a) Second singular value of the composite of $L$ projected random mixers (median of 30 draws), and the bound $\exp(-n\sum_l\delta(H_l))$ of Theorem~\ref{thm:contraction}(b) (dotted). (b) Vertex proximity under three updates, from the uniform matrix. The dashed line is proportional to $1/t$. (c) Time to leave the identity vertex, as a function of the initial proximity $\eps_0$. The dashed line is proportional to $1/\eps_0$. (d) Sinkhorn sweeps that give a column-sum error of $10^{-3}$, as a function of the horizon $\tau$ of the target.}
\label{fig:main}
\end{figure*}

%% file: sec_sinkhorn.tex
\section{Truncated Sinkhorn}\label{sec:sinkhorn}

mHC stops after $t_{\max}=20$ sweeps, and \citet{xie2026mhc} attribute the drift of the composite backward gain (up to about $1.6$ over 60 layers) to this truncation. Several follow-ups remove the truncation (Section~\ref{sec:intro}), but none of them measures its effect. We ask two questions. How large is the error of a truncated projection? How does this error grow with depth?

Fix $K>0$ with Sinkhorn scaling $H=\diag(\mu^\star)K\diag(\nu^\star)\in\Bno$. We describe the iteration by its column scaling $\nu>0$. After a row normalization, the current matrix is $\diag(\mu(\nu))K\diag(\nu)$ with $\mu(\nu):=\one\oslash(K\nu)$. One sweep (columns, then rows) maps $\nu$ to
\begin{equation}\label{eq:sweep}
\Psi(\nu):=\one\oslash\big(K\T\mu(\nu)\big).
\end{equation}
The fixed point is $\nu^\star$, and $\Psi(s\nu)=s\Psi(\nu)$ for $s>0$.

\begin{proposition}[Local convergence factor]\label{prop:knight}
Let $\nu=\nu^\star\odot(\one+v)$. Then
\[
\Psi(\nu)=\nu^\star\odot\big(\one+H\T Hv+O(\|v\|^2)\big),
\]
and the column sums of $\diag(\mu(\nu))K\diag(\nu)$ are $\one+(I-H\T H)v+O(\|v\|^2)$. On $\one^\perp$ the map $H\T H$ is symmetric with norm $\sigma_2(H)^2$. Thus, in the linearized iteration $v_k=(H\T H)^kv_0$, the column-sum error $r_k:=\|(I-H\T H)v_k\|_\infty$ satisfies $r_k\le\sigma_2(H)^{2k}\|Qv_0\|_2$. This bound is tight if $Qv_0$ has a component along a top singular vector of $HQ$ \citep{knight2008sinkhorn}. With the horizon $\tau=1/(1-\sigma_2(H))$,
\begin{equation}\label{eq:sweeps-per-efold}
e^{-2k/(\tau-1)}\ \le\ \sigma_2(H)^{2k}\ \le\ e^{-2k/\tau}.
\end{equation}
\end{proposition}

Thus one $e$-fold reduction of the error costs between $(\tau-1)/2$ and $\tau/2$ sweeps. A long memory (Corollary~\ref{cor:fading}) needs mixers with $\tau\gg1$, and Sinkhorn projects exactly these mixers most slowly. The next result shows how the remaining errors grow with depth.

\begin{proposition}[Drift of the backward gain]\label{prop:drift}
For $l<L$, let $H_l$ be nonnegative and row-stochastic, with column sums in $[1-e_l,1+e_l]$ and $e_l\le\frac12$. (Any number of sweeps that ends with a row normalization gives a nonnegative row-stochastic matrix.) Then $H_{L:l}\one=\one$, every column sum of $H_{L:l}$ is in $[e^{-2\sum_ie_i},e^{\sum_ie_i}]$, and $\|H_{L:l}\|_2\le e^{\frac12\sum_ie_i}$.
\end{proposition}

\begin{remark}[The budget limits the horizon]\label{rem:budget}
Let $r_0$ be the initial relative error of a layer, and let $\gamma$ be a target for the accumulated error $\sum_le_l$. Proposition~\ref{prop:knight} suggests that a layer with horizon $\tau_l$ needs about $\frac{\tau_l}{2}\log\frac{Lr_0}{\gamma}$ sweeps. Conversely, $t_{\max}$ sweeps give the accuracy $\gamma/L$ only to layers with $\tau_l\lesssim2t_{\max}/\log(Lr_0/\gamma)$. For mHC, $t_{\max}=20$, $L=60$, the observed drift $1.6=e^{0.47}$ gives $\gamma\approx0.47$, and $r_0\approx1$; thus the limit is $\tau\lesssim8$ layers. This is a heuristic scaling law, because Proposition~\ref{prop:knight} is asymptotic. Remark~\ref{rem:hilbert} gives a worst-case bound in Hilbert's projective metric.
\end{remark}

%% file: sec_isometric.tex
\section{Isometric mixers}\label{sec:isometric}

Proposition~\ref{prop:rigidity} explains a pattern in recent follow-ups. A design that avoids collapse and keeps the gains bounded must give up nonnegativity, and such designs use orthogonal or spectrally constrained mixers \citep{sengupta2026jpmhc,liu2026shc,guo2026ohc}. Let $U\in\R^{n\times(n-1)}$ be an orthonormal basis of $\one^\perp$. Then
\begin{equation}\label{eq:O1}
\On=\big\{J/n+UZU\T:\ Z\in O(n-1)\big\}.
\end{equation}
Thus $\On$ is a compact Lie group isomorphic to $O(n-1)$. It has dimension $(n-1)(n-2)/2$, and it contains $\Sn$.

\begin{proposition}[Isometric mixers]\label{prop:isometric}
Let $H_l\in\On$ for all $l$.
\begin{enumerate}[leftmargin=1.4em,itemsep=0pt,topsep=1pt,label=(\alph*)]
\item The recursions \eqref{eq:two-channel} hold, and $H_{L:l}$ preserves the norm of every stream difference. Thus the horizon is infinite.
\item All singular values of the residual-path Jacobian are one, at every depth. This is dynamical isometry of the residual path \citep{pennington2017resurrecting}.
\item For $A\in\On$, $T_A\On=\{A\Omega:\ \Omega\T=-\Omega,\ \Omega\one=0\}$. For the Frobenius metric, the Riemannian gradient of $f$ is $AQ\skewop(A\T\nabla f(A))Q$, with $\skewop(X):=\frac12(X-X\T)$. The Cayley map $R_A(A\Omega):=A(I-\frac12\Omega)^{-1}(I+\frac12\Omega)$ is a retraction on $\On$.
\end{enumerate}
\end{proposition}

Thus $\On$ has no boundary and no viscosity, a permutation is an ordinary point of the group, and the Cayley retraction needs only one linear solve. The cost is the loss of nonnegativity: the mixed stream is no longer a convex combination of the streams. The orthogonal mixers of \citet{guo2026ohc} are in $SO(n)$ and do not fix $\one$. Such a mixer couples the mean and the differences, so the mean is no longer a residual network with unit gain. The whole stream follows a rotated identity map instead: no norm is lost, but the mean reaches deep layers unchanged only up to a rotation. \citet{guo2026ohc} give evidence that this coupling helps; Remark~\ref{rem:ohc} gives the details.

%% file: sec_numerics.tex
\section{Numerical experiments}\label{sec:numerics}

All experiments use $n=4$ streams, as mHC does. They compute exact projections to machine precision, with a Newton solve of the balancing equations that we check against long Sinkhorn runs. The experiments check rates, in minutes on a laptop CPU; they do not train models. Random logits are uniform on $[-M,M]$. Appendix~\ref{app:figures} gives larger figures and one more experiment.

\paragraph{Collapse (Figure~\ref{fig:main}(a)).}
The composite converges to $J/n$ at a geometric rate for every logit scale $M$. The rate is faster than the bound of Theorem~\ref{thm:contraction}(b), but it has the same form. For small $M$ the bound is almost exact. At $M=4$ the factor per layer is about $0.55$, which is a horizon of about two layers. Random sharp logits do not give long memories; only mixers near a permutation do.

\paragraph{Approach (Figure~\ref{fig:main}(b)).}
The loss is $\ip{C}{H}$ with $C=J-I+0.1\,\Xi$, where $\Xi$ has uniform entries on $[0,1]$, and the start is $J/n$. Under logit gradient descent with step $0.2$, the proximity decays as $1/t$, as \eqref{eq:approach} predicts. The log--log slope over the last two decades is $-1.05$, and $t\,\eps(H(t))\approx0.26$ at $t=3000$. Under the straight-through update the proximity decays as $e^{-ct}$ with a fitted $c=0.98$, as \eqref{eq:fr-attained} predicts for off-permutation costs near one.

\paragraph{Escape (Figure~\ref{fig:main}(c)).}
The loss prefers a transposition, and the start is near the identity. Over four decades of $\eps_0$, the escape time of logit gradient descent grows as $0.65/\eps_0$ (log--log slope $0.96$), as \eqref{eq:escape} predicts. For the straight-through update it grows as $1.03\log(1/\eps_0)+1.2$. Adam on the logits (step $0.02$) escapes in $1.4$ to $3.3$ time units over the same range, as Remark~\ref{rem:adam} predicts.

\paragraph{Sinkhorn budget (Figure~\ref{fig:main}(d)).}
The targets are $(1-\tau^{-1})P_\pi+\tau^{-1}J/n$ with horizon $\tau$, after random diagonal scalings. The number of sweeps for a column-sum error of $10^{-3}$ grows linearly in the horizon, with a slope of about $1.5$ sweeps per layer of horizon. A target with $\tau=5$ needs 13 sweeps, $\tau=10$ needs 23, and $\tau=50$ needs 81. With $t_{\max}=20$, only targets with $\tau\lesssim7$ reach this accuracy, in agreement with Remark~\ref{rem:budget}. The error curves follow the local law $\sigma_2^{2k}$ from the first sweeps (Figure~\ref{fig:sinkhorn}).

%% file: sec_conclusion.tex
\section{Conclusion}\label{sec:conclusion}

The doubly stochastic constraint of mHC makes the stream mean a residual network with unit gain, for every mixer and every depth. The difference channel contracts by $\sigma_2$ per layer, so the extra width is a fading memory of $O(\tau)$ layers. Logit gradient flow moves near the permutations only at rate $1/t$, but the straight-through update moves at an exponential rate. Sinkhorn projects the mixers with a long horizon most slowly, so a fixed budget of sweeps limits the horizon. Appendix~\ref{app:discussion} discusses these results, their limits, and open problems.\markmainend

%% file: back_matter.tex
\section*{AI use statement}\addcontentsline{toc}{section}{AI use statement}

In this work, we used generative AI tools to search the literature, to draft and check proofs, to write the code of the numerical experiments, and to edit the text, including the layout of the paper and the rewrite of the proofs in structured form. We checked the proofs of all results and the code of all experiments. We take responsibility for the final content of this work, including text, claims, or artifacts produced with the aid of generative AI.

%% file: app_roadmap.tex
\section*{Roadmap}\label{app:roadmap}\addcontentsline{toc}{section}{Roadmap of the appendix}

The appendix has seven parts. Table~\ref{tab:roadmap} lists them, and Figure~\ref{fig:dependency} shows how the results depend on each other.

\begin{center}\small
\captionof{table}{Roadmap of the appendix.}\label{tab:roadmap}
\begin{tabular}{@{}l>{\raggedright\arraybackslash}p{10.6cm}>{\raggedright\arraybackslash}p{3.2cm}@{}}
\toprule
part & contents & supports\\
\midrule
\ref{app:related} & More related work, and a table of the notation & all sections\\
\ref{app:extra} & Additional results and remarks: the truncated three-way identity, the constant $4n^3$, a worst-case bound for truncated Sinkhorn, convergence on $\On$, mean preservation versus full orthogonality, and the convex relaxation & Sections~\ref{sec:geometry}--\ref{sec:isometric}\\
\ref{app:figures} & Additional figures, with the details of all experiments & Section~\ref{sec:numerics}\\
\ref{app:aux} & Known results that the proofs use, with references & Appendix~\ref{app:proofs}\\
\ref{app:prelim} & Preliminary lemmas, with structured proofs, and a guide to the proof format & Appendix~\ref{app:proofs}\\
\ref{app:proofs} & Structured proofs of all results of Sections~\ref{sec:signal}--\ref{sec:isometric} and Appendix~\ref{app:extra} & all results\\
\ref{app:discussion} & Discussion: what the theory says about mHC, its limits, and open problems & Section~\ref{sec:conclusion}\\
\bottomrule
\end{tabular}
\end{center}

\begin{figure}[t]\centering
\includegraphics[width=\linewidth]{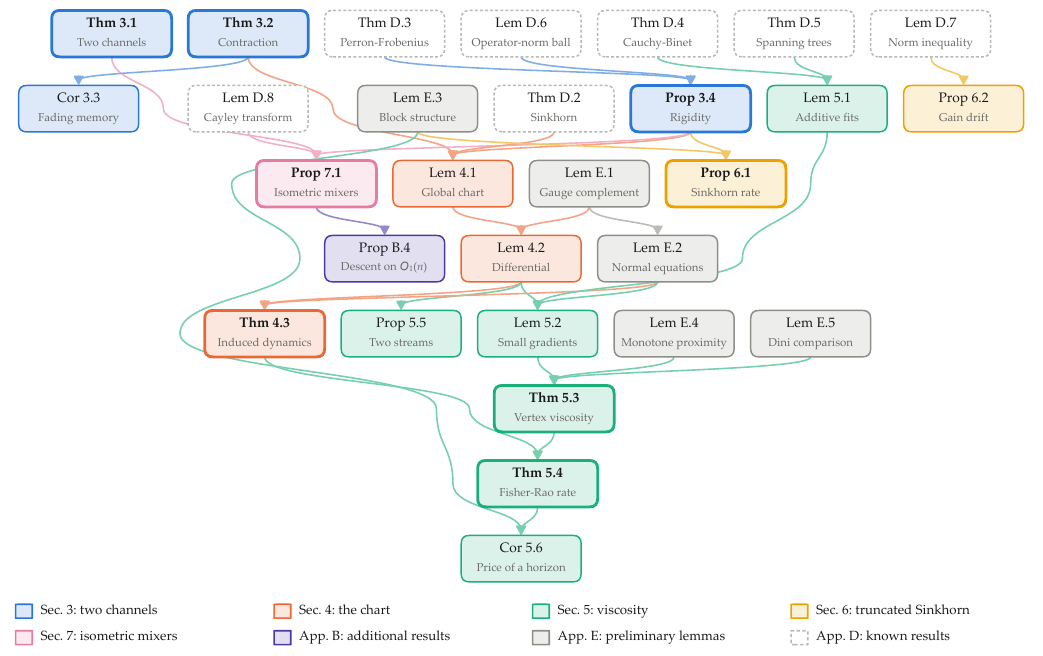}
\caption{Dependency graph of the results. An arrow from $X$ to $Y$ means that the proof of $Y$ cites $X$. We omit an arrow if other arrows already give a path from $X$ to $Y$. Bold frames mark the main results, and dashed frames mark the known results of Appendix~\ref{app:aux}. A script (\texttt{fig\_dependency.py}) builds the graph from the references in the structured proofs, so the graph and the proofs agree.}\label{fig:dependency}
\end{figure}

%% file: app_extra.tex
\section{More related work and notation}\label{app:related}

\subsection{More related work}

\paragraph{Hyper-connections and their follow-ups.}
Hyper-connections \citep{zhu2025hyper} belong to a family of designs that widen or densify the residual stream \citep{pagliardini2024denseformer,xiao2025muddformer,mak2025residual,heddes2025deepcrossattention,menghani2025laurel}. mHC \citep{xie2026mhc} constrains the mixer to the Birkhoff polytope. Its follow-ups replace the Sinkhorn loop by exact parameterizations of the polytope \citep{yang2026mhclite,zhou2026kromhc,lyubinin2026tbpmhc,dandachi2026gomhc,gu2026temper} or by an exact projection with implicit differentiation \citep{wang2026accelerating}. Other work measures how trained models use the streams \citep{alimaskina2026stream,peng2026ablate,zhao2026howdoes}. Orthogonal and spectral-sphere designs \citep{sengupta2026jpmhc,liu2026shc,guo2026ohc} give up nonnegativity. Implicit differentiation of the Sinkhorn fixed point, as in \citet{sengupta2026jpmhc,wang2026accelerating}, computes the differential that Lemma~\ref{lem:differential} identifies geometrically.

\paragraph{Concurrent work.}
\citet{guo2026ohc} appeared while we prepared this paper. They observe the two-channel decomposition and the contraction of a single layer. They also note that a doubly stochastic isometry is a permutation, and they identify the mean-preserving subgroup of $SO(n)$. Sections~\ref{sec:signal} and~\ref{sec:isometric} and Remark~\ref{rem:ohc} credit these points. The rates of Theorem~\ref{thm:contraction} and Corollary~\ref{cor:fading}, Proposition~\ref{prop:rigidity} except part (b), and the results of Sections~\ref{sec:geometry}--\ref{sec:sinkhorn} do not appear in that work.

\paragraph{Signal propagation.}
\citet{dong2021attention} and \citet{noci2022signal} study signal propagation and rank collapse in deep attention networks. \citet{biccari2026interior} apply contraction coefficients of Dobrushin type to row-stochastic attention rollouts. Our collapse results are deterministic, concern the stream dimension, and use both marginals of the mixer in an essential way. Dynamical isometry \citep{pennington2017resurrecting} is the property that Proposition~\ref{prop:isometric}(b) gives for isometric mixers.

\paragraph{Matrix scaling and its geometry.}
Matrix scaling and the Sinkhorn--Knopp algorithm have a large literature \citep{sinkhorn1964relationship,sinkhorn1967concerning,franklin1989scaling,knight2008sinkhorn,idel2016review,cuturi2013sinkhorn,altschuler2017near,peyre2019computational}. \citet{douik2019manifold} introduce the Fisher--Rao geometry of the positive doubly stochastic matrices and the Sinkhorn retraction. \citet{yoshizawa2026infogeom} studies degeneracies of the Fisher metric at the boundary of the Birkhoff polytope from a different point of view. A separate line of work views the Sinkhorn \emph{iteration} as mirror descent. Our object is different: the learning dynamics of logits \emph{through} a Sinkhorn layer.

\paragraph{Learning permutations.}
PermELearn \citep{helmbold2009learning} learns permutations by entropic mirror descent, with Sinkhorn balancing as the Bregman projection. \citet{adams2011ranking} and \citet{mena2018learning} learn permutations with Sinkhorn-normalized matrices. \citet{liu2026premature} analyzes premature commitment to a vertex under annealed Sinkhorn.

\paragraph{Optimization.}
Sections~\ref{sec:geometry} and~\ref{sec:viscosity} use four areas of optimization. They are mirror descent \citep{nemirovski1983problem,beck2003mirror}, natural gradient \citep{amari1998natural,amari2000methods}, the replicator equation as the Fisher--Rao flow on the simplex \citep{shahshahani1979new,hofbauer1998evolutionary}, and the rates of softmax and natural policy gradient \citep{mei2020global,li2021softmax}. Orthogonal recurrent networks \citep{arjovsky2016unitary} and Riemannian optimization with orthogonality constraints \citep{wen2013feasible,absil2008optimization,boumal2023introduction} give the background of Section~\ref{sec:isometric}.

\subsection{Notation}

\begin{center}\small
\captionof{table}{Notation.}\label{tab:notation}
\begin{tabular}{@{}lp{11.8cm}@{}}
\toprule
symbol & meaning\\
\midrule
$n$, $d$, $L$ & number of streams, model width, depth\\
$x_l\in\R^{n\times d}$; $\bar x_l$, $x_l^\perp$ & stream at layer $l$; its stream mean and its stream differences\\
$H_l$, $\hpre_l$, $\hpost_l$ & mixer, read-out, and write-in of layer $l$; $u_l=\mathcal F_l(\hpre_lx_l)$ is the block output\\
$H_{L:l}=H_{L-1}\cdots H_l$ & composite mixer (the Jacobian of the residual path)\\
$\one$, $J=\one\one\T$, $Q=I-J/n$ & all-ones vector, all-ones matrix, orthogonal projector onto $\one^\perp$\\
$\Bn$, $\Bno$, $\Sn$ & Birkhoff polytope, its positive part, permutation matrices $P_\pi$\\
$\Cn$, $\On$ & mean-preserving nonexpansive mixers; collapse-free elements of $\Cn$ ($\On\cong O(n-1)$)\\
$\delta(H)$, $\sigma_k(H)$ & minimum entry; $k$-th singular value, with $\sigma_2(H)=\|HQ\|_2$\\
$\eps_\pi(H)$, $\eps(H)$ & largest entry off the support of $\pi$; vertex proximity $\min_\pi\eps_\pi(H)$\\
$\tau=1/(1-\sigma_2)$ & horizon\\
$\Tsp$, $\gauge=\Tsp^\perp$ & matrices with zero row and column sums (the tangent space); gauge $\{a\one\T+\one b\T\}$\\
$g_H$, $\Pfr_H$, $\gradFR f$ & Fisher--Rao metric, $g_H$-orthogonal projection onto $\Tsp$, Fisher--Rao gradient\\
$\SK$, $\Phi(W)=\SK(\exp W)$ & Sinkhorn projection; Sinkhorn-logit chart\\
$A_H=\Pfr_H\circ(H\odot\cdot)$ & inverse Fisher--Rao tensor, $\dd\Phi_W=A_H$\\
$R_H(\xi)=\SK(H\odot e^{\xi\oslash H})$ & Sinkhorn retraction\\
$H^{(k)}$, $\eta$, $G$ & $k$-th training iterate; step size; bound on $\|\nabla f\|_\infty$\\
$\mu,\nu$; $\Psi$ & row and column scaling vectors; one Sinkhorn sweep \eqref{eq:sweep}\\
$e_l$, $r_k$ & column-sum error of layer $l$; linearized column-sum error after $k$ sweeps\\
$\odot$, $\oslash$; $\ip{\cdot}{\cdot}$; $\|\cdot\|_\infty$ & entrywise product and quotient; Frobenius inner product; entrywise maximum norm\\
\bottomrule
\end{tabular}
\end{center}

\section{Additional results and remarks}\label{app:extra}

\subsection{Truncated projections}

\begin{remark}[Truncated three-way identity]\label{rem:truncated-three-way}
Replace $\SK$ by the truncated projection $\SK_{t_{\max}}$ in \eqref{eq:three-way}. Then each of the three expressions differs from the exact mirror step only by the truncation error of Section~\ref{sec:sinkhorn}. The map $K\mapsto\SK_{t_{\max}}(D_1KD_2)$ does not depend on $D_2$, because it starts with a column normalization. It depends on $D_1$, through the first column normalization. Thus the three truncated iterates are close to each other, but they are not equal.
\end{remark}

\subsection{The constant in Theorem~\ref{thm:viscosity}}

\begin{remark}[The constant $4n^3$]\label{rem:constant}
The proof of Theorem~\ref{thm:viscosity} gives the constant $4n^2(n-1)$, and we do not claim that this constant is sharp. Lemma~\ref{lem:additive} is sharp over all positive weights. But the proof uses it with doubly stochastic weights. For such weights, numerical searches give smaller fitted values: about $3.9\,\|C\|_\infty$ for $n=3$ and $n=4$, against $2n-1=5$ and $7$. The exponent is sharp: the rate $1/t$ is attained for $n=2$ (Proposition~\ref{prop:n2}).
\end{remark}

\subsection{A worst-case bound for truncated Sinkhorn}

\begin{remark}[Bounded logits]\label{rem:hilbert}
The Sinkhorn map contracts Hilbert's projective metric, and this fact gives a global, non-asymptotic bound. Each half-sweep multiplies the projective distance between scaling vectors by at most Birkhoff's coefficient $\kappa(K)=(\sqrt{\theta(K)}-1)/(\sqrt{\theta(K)}+1)$ \citep{birkhoff1957extensions}. Here $\theta(K)=\max_{i,j,k,l}K_{ik}K_{jl}/(K_{jk}K_{il})$ is the largest cross-ratio. \citet{franklin1989scaling} use this fact to prove linear convergence of Sinkhorn--Knopp in that metric \citep[Theorems~4.1--4.2]{peyre2019computational}. For $K=\exp W$ with $\|W\|_\infty\le M$ we have $\theta(K)\le e^{4M}$, and therefore
\begin{equation}\label{eq:tanh}
\kappa(K)\le\frac{e^{2M}-1}{e^{2M}+1}=\tanh M .
\end{equation}
This is a worst-case factor of $\tanh(M)^2$ per sweep. Sharp logits (large $M$) are necessary to reach mixers near the vertices, because $1-\sigma_2\ge e^{-4M}$ by Theorem~\ref{thm:contraction}(d). By \eqref{eq:tanh}, the same sharpness slows the projection in the worst case: $1-\tanh(M)^2\approx4e^{-2M}$. For random logits the truncation error is much smaller than this worst case (Figure~\ref{fig:sinkhorn}). The local law of Proposition~\ref{prop:knight} is the relevant one, and it is tight.
\end{remark}

\subsection{Isometric mixers}

\begin{proposition}[Convergence on $\On$]\label{prop:rgd-conv}
Let $f$ be twice continuously differentiable on a neighborhood of $\On$, and write $\gradFR_{\mathcal O}f$ for its Riemannian gradient on $\On$ (Proposition~\ref{prop:isometric}(c)). For a sufficiently small constant step $\eta$, Riemannian gradient descent $A^{(k+1)}=R_{A^{(k)}}(-\eta\,\gradFR_{\mathcal O}f(A^{(k)}))$ satisfies $\min_{k<N}\|\gradFR_{\mathcal O}f(A^{(k)})\|_F=O(N^{-1/2})$ \citep[Chapter~4]{boumal2023introduction}.
\end{proposition}

Three properties separate $\On$ from the Sinkhorn-logit design on $\Bno$. First, $\On$ has no boundary, so it has no viscosity. A permutation is an ordinary point, and no loss is flat there for structural reasons (compare \eqref{eq:flat}). Second, there is no projection loop: the Cayley retraction is one $(n-1)\times(n-1)$ linear solve, so Section~\ref{sec:sinkhorn} does not apply. Third, the number of parameters is $(n-1)(n-2)/2$, not $(n-1)^2$. The price is that entries of $A$ can be negative. Then the $\ell_1$ row-sum diagnostic of \citet{xie2026mhc} can exceed one, although every $\ell_2$ gain is one. The read-out and the write-in in mHC stay nonnegative.

\begin{remark}[Mean preservation versus full orthogonality]\label{rem:ohc}
The orthogonal hyper-connections of \citet{guo2026ohc} use mixers in $SO(n)$ that do not fix $\one$. The authors identify the mean-preserving subgroup $\{A\in SO(n):A\one=\one\}\cong SO(n-1)$. With an intervention experiment, they argue that the value of $SO(n)$ is the coupling between the mean and the differences that this subgroup forbids. Theorem~\ref{thm:two-channel} makes the trade explicit.

In the basis $[\one/\sqrt n,\ U]$, a general orthogonal mixer has off-diagonal blocks that rotate the mean into the differences and back. Thus the stream mean is no longer a residual network with unit gain. Instead the whole stream obeys a \emph{rotated} identity map, $x_L=H_{L:l}x_l+\sum_iH_{L:i+1}(\hpost_i)\T u_i$ with $H_{L:l}$ orthogonal. No norm is lost in either direction. But mHC restores one property: the mean signal of layer $l$ reaches layer $L$ unchanged. Under a general orthogonal mixer, this property holds only up to an orthogonal change of coordinates. It is an empirical question if this property is worth its cost. The evidence of \citet{guo2026ohc} favors the coupled version. By Proposition~\ref{prop:rigidity}, the question does not arise for nonnegative designs, because their only collapse-free mixers are the permutations.
\end{remark}

\begin{remark}[The convex relaxation]\label{rem:convex-relaxation}
The set $\Cn=\conv\On$ of mean-preserving nonexpansive mixers is convex (Proposition~\ref{prop:rigidity}(a)), and its Frobenius projection is explicit: $A\mapsto J/n+U\,\mathrm{clip}(U\T AU)\,U\T$. Here $\mathrm{clip}$ replaces every singular value larger than one by one. Thus projected gradient descent on $\Cn$ is available. But the interior of $\Cn$ consists of contractions, and they collapse exactly as the interior of $\Bn$ does. The difference is in the collapse-free part. For $\Cn$ it is the connected manifold $\On$, on which Riemannian methods can move; for $\Bn$ it is the finite set $\Sn$. The spectral-sphere design of \citet{liu2026shc} is a relaxation of the same kind for the full orthogonal group.
\end{remark}

\section{Additional figures}\label{app:figures}

All experiments use $n=4$ and exact projections. A warm-started Newton solve of the balancing equations computes each projection, and long Sinkhorn runs check it (maximum difference $1.2\times10^{-12}$). The supplementary material contains the code. The script \texttt{figs\_appendix.py} computes Figures~\ref{fig:collapse}, \ref{fig:sinkhorn}, and~\ref{fig:channels}. \texttt{escape.py} computes the escape times, \texttt{figs.py} computes the approach curves and the Sinkhorn budget, \texttt{figs\_main.py} draws Figures~\ref{fig:main} and~\ref{fig:viscosity}, and \texttt{fig\_teaser.py} computes and draws Figure~\ref{fig:teaser}.

\begin{figure}[h]
\centering
\includegraphics[width=0.92\linewidth]{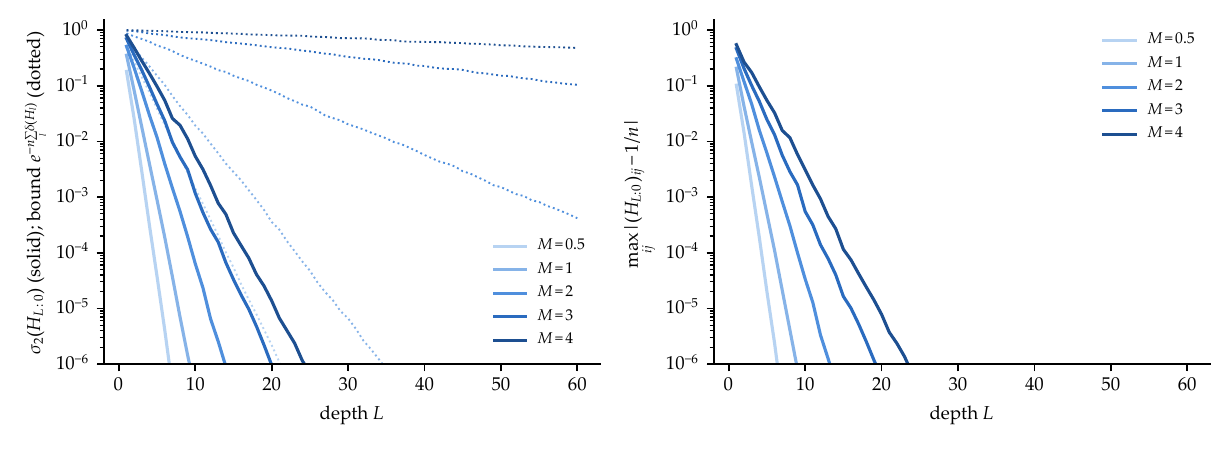}
\caption{Collapse of composite mixers (Theorem~\ref{thm:contraction}). Left: $\sigma_2$ of the composite of independent projected random logits of scale $M$ (median of 30 draws, solid), and the bound $\exp(-n\sum_l\delta(H_l))$ (dotted). Right: largest entrywise distance of the composite to $J/n$. Even for $M=4$, the difference channel loses six orders of magnitude within 25 layers.}
\label{fig:collapse}
\end{figure}

\begin{figure}[h]
\centering
\includegraphics[width=0.92\linewidth]{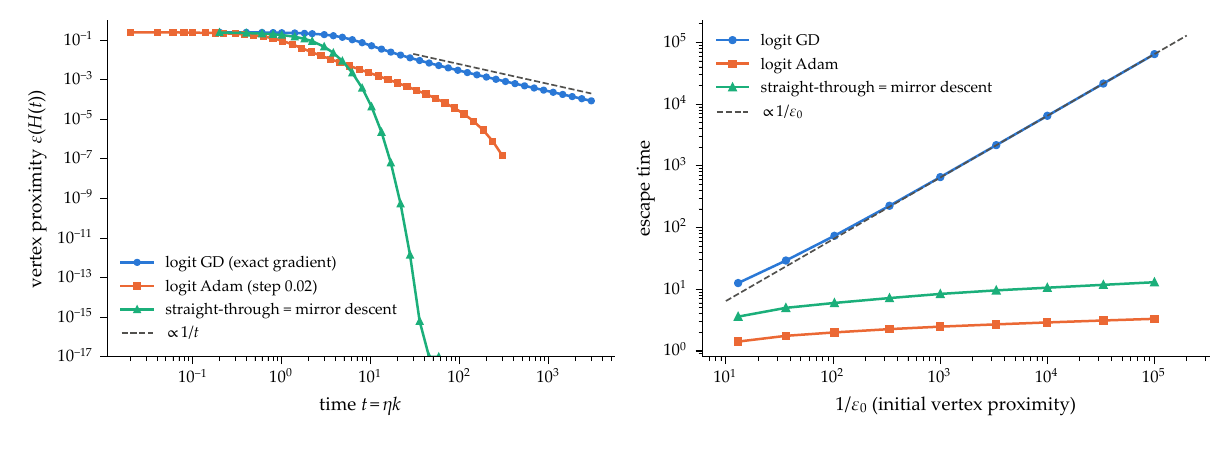}
\caption{Vertex viscosity (Theorems~\ref{thm:viscosity} and~\ref{thm:fr-rate}). Left: vertex proximity for the loss $\ip{C}{H}$, $C=J-I+0.1\,\Xi$, from the uniform matrix. The updates are exact logit gradient descent (step $0.2$), Adam on the logits (step $0.02$), and the straight-through update (step $0.2$). The dashed line is proportional to $1/t$. Right: time to leave the identity vertex for a transposition vertex, as a function of the initial proximity $\eps_0$, under the same three updates. The start is $W=\log(1/\eps_0)I$, and the escape time is the first time at which every entry on the transposition is larger than $0.5$ and every other entry is smaller than $0.25$. The dashed line is proportional to $1/\eps_0$.}
\label{fig:viscosity}
\end{figure}

\begin{figure}[h]
\centering
\includegraphics[width=0.92\linewidth]{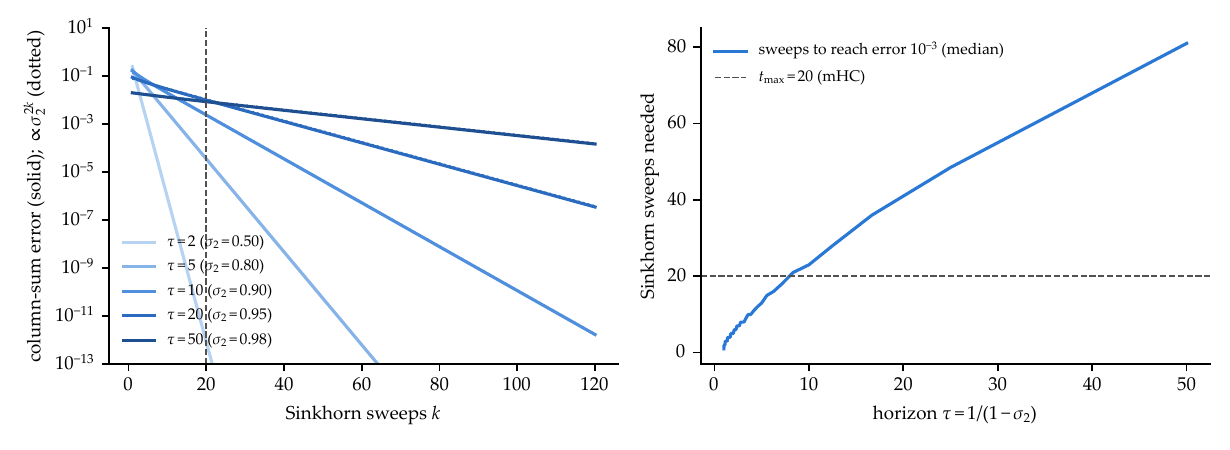}
\caption{Truncated Sinkhorn and the horizon (Proposition~\ref{prop:knight}). Left: column-sum error of the row-normalized iterate against the number of sweeps, for targets $(1-\tau^{-1})P_\pi+\tau^{-1}J/n$ with horizon $\tau$ and a 4-cycle $\pi$. Each run starts from a random diagonal scaling of the target (solid). The dotted lines show the local law $\propto\sigma_2^{2k}$, and the vertical line is $t_{\max}=20$. Right: median number of sweeps (20 random scalings) for a column-sum error of $10^{-3}$, as a function of $\tau$.}
\label{fig:sinkhorn}
\end{figure}

\begin{figure}[h]
\centering
\includegraphics[width=0.5\linewidth]{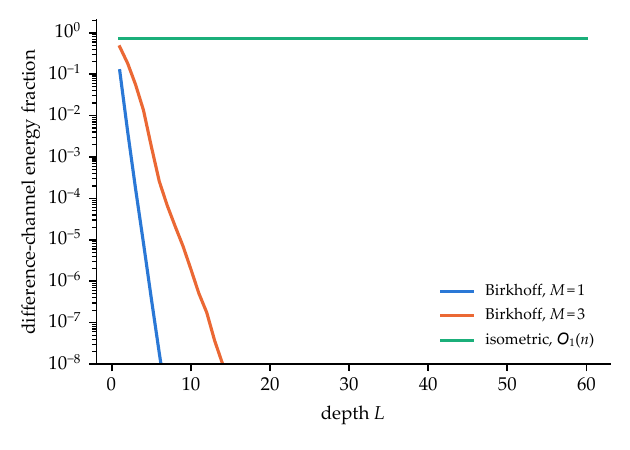}
\caption{Two channels (Theorem~\ref{thm:two-channel} and Proposition~\ref{prop:isometric}). Fraction of the energy of the residual-path signal $H_{L:0}x$ in the difference channel, for a random Gaussian input $x\in\R^{4\times64}$. The mixers are projected random logits ($M=1,3$) or random isometric mixers in $\On$ (median of 10 draws). Under doubly stochastic mixing the energy of the differences decays geometrically. Under isometric mean-preserving mixing it stays constant, because every factor is an isometry of both channels.}
\label{fig:channels}
\end{figure}

%% file: app_aux.tex
\section{Auxiliary results}\label{app:aux}

This appendix collects the known results that the proofs use, in the form in which we use them. Nothing in this appendix is new.

\begin{theorem}[Birkhoff--von Neumann]\label{aux:bvn}
$\Bn=\conv\Sn$, and the permutation matrices are exactly the extreme points of $\Bn$ \citep{birkhoff1946tres}.
\end{theorem}

\begin{theorem}[Sinkhorn]\label{aux:sinkhorn}
Let $K\in\R^{n\times n}$ have positive entries \citep{sinkhorn1964relationship,sinkhorn1967concerning}.
\begin{enumerate}[leftmargin=1.6em,itemsep=1pt,label=(\alph*)]
\item There are positive diagonal matrices $D_1,D_2$ such that $D_1KD_2$ is doubly stochastic. The matrix $D_1KD_2$ is unique, and the pair $(D_1,D_2)$ is unique up to $(sD_1,s^{-1}D_2)$ with $s>0$.
\item The alternating normalization of rows and columns, started at $K$, converges to $D_1KD_2$.
\end{enumerate}
\end{theorem}

\begin{theorem}[Perron--Frobenius]\label{aux:pf}
Let $N\in\R^{n\times n}$ be nonnegative.
\begin{enumerate}[leftmargin=1.6em,itemsep=1pt,label=(\alph*)]
\item If $N$ is irreducible, then its spectral radius is a simple eigenvalue \citep[Chapter~1]{seneta1981nonnegative}, \citep[Chapter~8]{horn2013matrix}.
\item If $N$ is symmetric and reducible, then there is a nonempty proper subset $\mathcal S\subset[n]$ with $N_{jk}=0$ for all $j\in\mathcal S$ and $k\notin\mathcal S$. (A matrix is reducible if a simultaneous permutation of its rows and columns gives a block upper triangular form. For a symmetric matrix the off-diagonal block is then zero.)
\end{enumerate}
\end{theorem}

\begin{theorem}[Cauchy--Binet formula and Cramer's rule]\label{aux:cb}
Let $X,Y\in\R^{m\times p}$ with $m\ge p$, and let $D=\diag(\omega)\in\R^{m\times m}$. Then
\[
\det(X\T DY)=\sum_{S\subseteq[m],\ |S|=p}\omega_S\,\det(X_S)\det(Y_S),\qquad \omega_S:=\prod_{s\in S}\omega_s ,
\]
where $X_S$ is the $p\times p$ submatrix of $X$ with the rows in $S$ \citep[Section~0.8.7]{horn2013matrix}. If $X\in\R^{p\times p}$ is nonsingular and $Xz=y$, then $z_r=\det(X^{(r)})/\det(X)$, where $X^{(r)}$ is $X$ with its $r$-th column replaced by $y$ (Cramer's rule).
\end{theorem}

\begin{theorem}[Spanning trees and incidence matrices]\label{aux:tree}
Let $(V,E)$ be a connected graph, and give its edges arbitrary orientations. Let $B\in\R^{|V|\times|E|}$ be the signed incidence matrix: $B_{ve}=+1$ if $v$ is the head of $e$, $B_{ve}=-1$ if $v$ is the tail of $e$, and $B_{ve}=0$ otherwise. Let $B_0$ be $B$ with one row removed. For a set $T\subseteq E$ of $|V|-1$ edges, the square submatrix of $B_0$ with the columns in $T$ is nonsingular if and only if $T$ is a spanning tree of the graph. In that case its determinant is $\pm1$ \citep[Chapter~2]{bapat2014graphs}.
\end{theorem}

\begin{lemma}[Operator-norm ball]\label{aux:convO}
The closed unit ball of the operator norm on $\R^{m\times m}$ is the convex hull of the orthogonal group $O(m)$.
\end{lemma}
\begin{proof}
\begin{steps}
\step{1}{1} $\conv O(m)$ is a subset of the ball.
\pf Orthogonal matrices have norm one, and the ball is convex.
\step{1}{2} \Assume $\|B\|_2\le1$. \Prove $B\in\conv O(m)$.
\pf Write $B=V\Sigma Z\T$ with $V,Z\in O(m)$ and $\Sigma=\diag(\sigma_1,\dots,\sigma_m)$, $\sigma_i\in[0,1]$. For a sign vector $s\in\{\pm1\}^m$, let $\lambda_s:=\prod_i\frac{1+s_i\sigma_i}{2}\ge0$. Then $\sum_s\lambda_s=1$ and $\sum_s\lambda_s\,s_i=\sigma_i$ for every $i$. Thus $\Sigma=\sum_s\lambda_s\diag(s)$, and $B=\sum_s\lambda_s\,V\diag(s)Z\T$ is a convex combination of orthogonal matrices.
\qedstep{1}{3}
\pf By \sref{1}{1} and \sref{1}{2}.
\end{steps}
\end{proof}

\begin{lemma}[A norm inequality]\label{aux:schur}
For $A\in\R^{n\times n}$, $\|A\|_2^2\le\|A\|_1\|A\|_\infty$, where $\|A\|_1$ is the largest absolute column sum and $\|A\|_\infty$ is the largest absolute row sum \citep[Chapter~5]{horn2013matrix}.
\end{lemma}

\begin{lemma}[Cayley transform]\label{aux:cayley}
Let $\Omega\in\R^{n\times n}$ be skew-symmetric. Then $I-\frac12\Omega$ is invertible, the matrix $C(\Omega):=(I-\frac12\Omega)^{-1}(I+\frac12\Omega)$ is orthogonal, and $\Omega\mapsto C(\Omega)$ is smooth with $C(0)=I$ and $\frac{\dd}{\dd t}C(t\Omega)|_{t=0}=\Omega$ \citep{wen2013feasible}.
\end{lemma}

\begin{theorem}[Hilbert's projective metric]\label{aux:hilbert}
Let $K>0$, and let $\theta(K):=\max_{i,j,k,l}K_{ik}K_{jl}/(K_{jk}K_{il})$ and $\kappa(K):=(\sqrt{\theta(K)}-1)/(\sqrt{\theta(K)}+1)$. The map $x\mapsto Kx$ multiplies Hilbert's projective metric on the positive orthant by at most $\kappa(K)$ \citep{birkhoff1957extensions}. Thus the Sinkhorn scaling vectors converge linearly in this metric, with a factor of at most $\kappa(K)^2$ per sweep \citep{franklin1989scaling}, \citep[Theorems~4.1--4.2]{peyre2019computational}.
\end{theorem}

\begin{theorem}[Knight]\label{aux:knight}
Let $K>0$ have the doubly stochastic scaling $H$. The Sinkhorn--Knopp iteration converges linearly, with asymptotic rate $\sigma_2(H)^2$ per sweep \citep{knight2008sinkhorn}; see also \citet[Section~7.2]{idel2016review}. Proposition~\ref{prop:knight} gives a self-contained derivation.
\end{theorem}

\section{Preliminary lemmas}\label{app:prelim}

\paragraph{How to read the proofs.}
We write the proofs in the structured style of \citet{lamport1995write,lamport2012write}. A proof is a list of numbered steps, and $\pflab{k}{j}$ is the $j$-th step at level $k$. Each step states a claim, and a proof of the claim follows it. The proof of a step is a short paragraph or a list of steps at the next level. The last step of each list, \textsc{q.e.d.}, proves the goal of the list from the earlier steps. \textsc{Assume} and \textsc{Prove} state the hypothesis and the goal of a step. \textsc{Suffices} reduces the current goal to a new goal. \textsc{Let} introduces a symbol. A reader who wants only the main idea can read the level-1 steps and skip their proofs.

\begin{lemma}[Gauge complement]\label{lem:gauge-complement}
Let $H\in\Bno$.
\begin{enumerate}[leftmargin=1.6em,itemsep=1pt,label=(\alph*)]
\item The $g_H$-orthogonal complement of $\Tsp$ in $\R^{n\times n}$ is $H\odot\gauge$, and $\R^{n\times n}=\Tsp\oplus(H\odot\gauge)$.
\item For every $\xi\in\R^{n\times n}$, the affine space $\xi+H\odot\gauge$ meets $\Tsp$ in exactly one point, namely $\Pfr_H(\xi)$. This point is given by \eqref{eq:FRproj}.
\item $\ker\Pfr_H=H\odot\gauge$.
\end{enumerate}
\end{lemma}
\begin{proof}
\begin{steps}
\step{1}{1} A matrix $\xi$ is $g_H$-orthogonal to $\Tsp$ if and only if $\xi\in H\odot\gauge$.
\pf By \eqref{eq:FR}, $g_H(\xi,\zeta)=\ip{\xi\oslash H}{\zeta}$. Thus $\xi$ is $g_H$-orthogonal to $\Tsp$ if and only if $\xi\oslash H\in\Tsp^\perp=\gauge$, that is, $\xi\in H\odot\gauge$.
\step{1}{2} $\R^{n\times n}=\Tsp\oplus(H\odot\gauge)$, and the sum is $g_H$-orthogonal.
\pf The map $s\mapsto H\odot s$ is injective because $H>0$. Thus $\dim(H\odot\gauge)=\dim\gauge=2n-1$, and $\dim\Tsp+\dim(H\odot\gauge)=(n-1)^2+2n-1=n^2$. By \sref{1}{1} the two subspaces are $g_H$-orthogonal, so they meet only in $0$.
\step{1}{3} Let $\xi\in\R^{n\times n}$. Then $\Pfr_H(\xi)=\xi-H\odot(a\one\T+\one b\T)$, where $(a,b)$ solves \eqref{eq:FRproj}.
\pf By \sref{1}{2}, $\xi=\Pfr_H(\xi)+H\odot s$ for a unique $s=a\one\T+\one b\T\in\gauge$. The condition $\Pfr_H(\xi)\in\Tsp$ says that the row sums and the column sums of $\xi-H\odot s$ are zero:
\[
\textstyle\sum_j\xi_{ij}-\sum_jH_{ij}(a_i+b_j)=0,\qquad \sum_i\xi_{ij}-\sum_iH_{ij}(a_i+b_j)=0 .
\]
The rows and the columns of $H$ sum to one, so these equations are \eqref{eq:FRproj}.
\qedstep{1}{4}
\pf Steps \sref{1}{1} and \sref{1}{2} give (a). By \sref{1}{2}, the affine space $\xi+H\odot\gauge$ meets $\Tsp$ in exactly one point, and by \sref{1}{3} this point is $\Pfr_H(\xi)$; this gives (b). By \sref{1}{2}, $\Pfr_H(\xi)=0$ if and only if $\xi\in H\odot\gauge$; this gives (c).
\end{steps}
\end{proof}

\begin{lemma}[Balancing equations are normal equations]\label{lem:fit-normal}
Let $H\in\Bno$, $C\in\R^{n\times n}$, and $(a,b)\in\R^n\times\R^n$. The following are equivalent.
\begin{enumerate}[leftmargin=1.6em,itemsep=1pt,label=(\roman*)]
\item $a_i+\sum_jH_{ij}b_j=\sum_jH_{ij}C_{ij}$ for all $i$, and $b_j+\sum_iH_{ij}a_i=\sum_iH_{ij}C_{ij}$ for all $j$.
\item $(a,b)$ minimizes $F(a,b):=\sum_{i,j}H_{ij}\,(C_{ij}-a_i-b_j)^2$.
\end{enumerate}
Consequently $\Pfr_H(H\odot C)=H\odot(C-m)$, where $m=a\one\T+\one b\T$ is the weighted least-squares fit of $C$ by an additive model with weights $H$.
\end{lemma}
\begin{proof}
\begin{steps}
\step{1}{1} The minimizers of $F$ are exactly its stationary points.
\pf $F$ is a convex quadratic function of $(a,b)$.
\step{1}{2} The stationary points of $F$ are exactly the solutions of (i).
\pf The partial derivatives of $F$ are zero if and only if
\[
\textstyle\sum_jH_{ij}(C_{ij}-a_i-b_j)=0\ \ (i\in[n]),\qquad
\sum_iH_{ij}(C_{ij}-a_i-b_j)=0\ \ (j\in[n]).
\]
Because $\sum_jH_{ij}=1=\sum_iH_{ij}$, these equations are the equations in (i).
\qedstep{1}{3}
\pf Steps \sref{1}{1} and \sref{1}{2} give the equivalence. The equations (i) are \eqref{eq:FRproj} with $\xi=H\odot C$. Thus Lemma~\ref{lem:gauge-complement}(b) gives $\Pfr_H(H\odot C)=H\odot C-H\odot m$.
\end{steps}
\end{proof}

\begin{lemma}[Block structure]\label{lem:two-block}
Let $H\in\Bn$, and let $[\one/\sqrt n,\ U]$ be an orthonormal basis of $\R^n$ ($U\T U=I$, $U\T\one=0$). Then $H=J/n+QHQ$, the subspaces $\R\one$ and $\one^\perp$ are invariant under $H$ and $H\T$, and in this basis
\[
H=\begin{pmatrix}1&0\\0&U\T HU\end{pmatrix}.
\]
Thus the singular values of $H$ are $1$ and the singular values of $U\T HU$. In particular $\sigma_2(H)=\|U\T HU\|_2$ and $|\det(U\T HU)|=\prod_{i=2}^n\sigma_i(H)$. The same statements hold for products of doubly stochastic matrices.
\end{lemma}
\begin{proof}
\begin{steps}
\step{1}{1} $H=J/n+QHQ$.
\pf By (F3), $HJ=J$ and $HQ=QH$. Thus $H=H(J/n+Q)=J/n+HQ$, and $HQ=HQ^2=QHQ$.
\step{1}{2} $\R\one$ and $\one^\perp$ are invariant under $H$ and $H\T$.
\pf $H\one=\one$. For $x\perp\one$, $\ip{Hx}{\one}=\ip{x}{H\T\one}=\ip{x}{\one}=0$. The same argument applies to $H\T\in\Bn$.
\step{1}{3} The singular values of $H$ are $1$ and the singular values $s_1\ge\dots\ge s_{n-1}$ of $U\T HU$, with $\sigma_{k+1}(H)=s_k$.
\pf By \sref{1}{2}, $H$ has the stated block form, and a block diagonal matrix has the singular values of its blocks. By (F1), $s_1\le\|H\|_2\le1$, so the largest singular value is the block $1$.
\qedstep{1}{4}
\pf Steps \sref{1}{1}--\sref{1}{3} give all claims, because $|\det(U\T HU)|=\prod_ks_k$. By (F2), products of doubly stochastic matrices are doubly stochastic.
\end{steps}
\end{proof}

\begin{lemma}[Monotonicity of the vertex proximity]\label{lem:eps-mono}
Let $H,H'\in\R^{n\times n}_{>0}$ and $0<\lambda\le\lambda'$ satisfy $\lambda H_{ij}\le H'_{ij}\le\lambda' H_{ij}$ for all $(i,j)$. Then $\lambda\,\eps_\pi(H)\le\eps_\pi(H')\le\lambda'\eps_\pi(H)$ for every $\pi$, and $\lambda\,\eps(H)\le\eps(H')\le\lambda'\eps(H)$.
\end{lemma}
\begin{proof}
\begin{steps}
\step{1}{1} The claim holds for $\eps_\pi$.
\pf $\eps_\pi$ is the maximum of the entries on a fixed set of cells, and a maximum is monotone and positively homogeneous.
\qedstep{1}{2}
\pf $\eps=\min_\pi\eps_\pi$ is a minimum of finitely many functions with the two properties of \sref{1}{1}, and a minimum keeps both properties.
\end{steps}
\end{proof}

\begin{lemma}[Dini comparison]\label{lem:dini}
Let $\psi:[0,T)\to\R$ be continuous, let $c\in\R$, and assume that
\[
D^+\psi(t):=\limsup_{h\downarrow0}\frac{\psi(t+h)-\psi(t)}{h}\ \le\ c\qquad\text{for every }t\in[0,T).
\]
Then $\psi(t)\le\psi(0)+ct$ for all $t\in[0,T)$.
\end{lemma}
\begin{proof}
\begin{steps}
\step{1}{1} \Suffices For every $\nu>0$, the function $\psi_\nu(t):=\psi(t)-(c+\nu)t$ is nonincreasing.
\pf Then $\psi(t)-(c+\nu)t\le\psi(0)$ for every $\nu>0$. Let $\nu\downarrow0$.
\step{1}{2} \Assume $\nu>0$, $0\le t_0<t_1<T$, and $\psi_\nu(t_1)>\psi_\nu(t_0)$. \Prove \textsc{false}.
\begin{steps}
\step{2}{1} \Let $s:=\sup\{t\in[t_0,t_1]:\ \psi_\nu(t)\le \psi_\nu(t_0)\}$. Then $s<t_1$, $\psi_\nu(s)=\psi_\nu(t_0)$, and $\psi_\nu(t)>\psi_\nu(t_0)$ for all $t\in(s,t_1]$.
\pf The set contains $t_0$, so $s$ exists. The set is closed because $\psi_\nu$ is continuous, so $\psi_\nu(s)\le \psi_\nu(t_0)$. By the definition of $s$, $\psi_\nu(t)>\psi_\nu(t_0)$ for $t\in(s,t_1]$. By continuity, $\psi_\nu(s)\ge \psi_\nu(t_0)$. Thus $\psi_\nu(s)=\psi_\nu(t_0)<\psi_\nu(t_1)$, and $s<t_1$.
\step{2}{2} $D^+\psi_\nu(s)\ge0$.
\pf By \sref{2}{1}, $\big(\psi_\nu(s+h)-\psi_\nu(s)\big)/h>0$ for $0<h\le t_1-s$.
\qedstep{2}{3}
\pf $D^+\psi_\nu=D^+\psi-(c+\nu)\le-\nu<0$ everywhere. This contradicts \sref{2}{2}.
\end{steps}
\qedstep{1}{3}
\pf By \sref{1}{2}, $\psi_\nu(t_1)\le \psi_\nu(t_0)$ whenever $t_0<t_1$. Thus \sref{1}{1} applies.
\end{steps}
\end{proof}

%% file: app_proofs.tex
\section{Proofs}\label{app:proofs}

\subsection{Proofs for Section~\ref{sec:signal}}

\begin{proof}[Proof of Theorem~\ref{thm:two-channel}]
\begin{steps}
\step{1}{1} $\bar x_{l+1}\T=\bar x_l\T+\bar h_l\,u_l$.
\pf Multiply \eqref{eq:hc} on the left by $\frac1n\one\T$. Because $\one\T H_l=\one\T$, we get $\frac1n\one\T H_lx_l=\bar x_l\T$ and $\frac1n\one\T(\hpost_l)\T u_l=\bar h_l\,u_l$.
\step{1}{2} $x^\perp_{l+1}=H_lx^\perp_l+Q(\hpost_l)\T u_l$.
\pf Multiply \eqref{eq:hc} on the left by $Q$. By (F3), $QH_l=H_lQ$, so $QH_lx_l=H_lx^\perp_l$.
\step{1}{3} The two recursions \eqref{eq:two-channel} determine \eqref{eq:hc}.
\pf They determine $x_{l+1}=\one\bar x_{l+1}\T+x^\perp_{l+1}$ by \eqref{eq:channels}.
\step{1}{4} For $l<L$,
\begin{equation}\label{eq:two-channel-unrolled}
\bar x_L\T=\bar x_l\T+\sum_{i=l}^{L-1}\bar h_i\,u_i,\qquad
x^{\perp}_L=H_{L:l}\,x^{\perp}_l+\sum_{i=l}^{L-1}H_{L:i+1}\,Q(\hpost_i)\T u_i .
\end{equation}
\pf Induction on $L$, with \sref{1}{1} and \sref{1}{2}.
\step{1}{5} The residual-path Jacobian is $x\mapsto H_{L:l}x$, its adjoint is $y\mapsto H_{L:l}\T y$, and both maps fix $\one$.
\pf Unroll \eqref{eq:hc}:
\[
x_L=H_{L:l}\,x_l+\sum_{i=l}^{L-1}H_{L:i+1}(\hpost_i)\T u_i .
\]
With the block outputs $u_i$ held fixed, this map is affine in $x_l$, and its linear part is $x\mapsto H_{L:l}x$. Its adjoint for the Frobenius inner product is $y\mapsto H_{L:l}\T y$. By (F2), $H_{L:l}\in\Bn$, so $H_{L:l}\one=\one$ and $H_{L:l}\T\one=\one$.
\qedstep{1}{6}
\pf Steps \sref{1}{1}--\sref{1}{3} give the equivalence, and \sref{1}{5} gives the statement about the Jacobian.
\end{steps}
\end{proof}

\begin{proof}[Proof of Theorem~\ref{thm:contraction}]
\begin{steps}
\step{1}{1} For every $H\in\Bn$, $\sigma_2(H)\le1-n\delta(H)$.
\begin{steps}
\step{2}{1} \Let $\delta:=\delta(H)$. \Case $n\delta=1$.
\pf Every entry is at least $1/n$, and every row sums to one. Thus $H=J/n$, and $\sigma_2(H)=\|JQ\|_2/n=0=1-n\delta$.
\step{2}{2} \Case $n\delta<1$. \Let $R:=(H-\delta J)/(1-n\delta)$. Then $R\in\Bn$.
\pf $R\ge0$ because $H_{ij}\ge\delta$. Each row sum and each column sum of $R$ is $(1-n\delta)/(1-n\delta)=1$.
\step{2}{3} \Case $n\delta<1$. Then $\|Hx\|_2\le(1-n\delta)\|x\|_2$ for every $x\perp\one$.
\pf For $x\perp\one$ we have $Jx=0$. Thus $Hx=(1-n\delta)Rx+\delta Jx=(1-n\delta)Rx$, and $\|R\|_2\le1$ by \sref{2}{2} and (F1).
\qedstep{2}{4}
\pf The two cases cover all $H$, because $n\delta\le1$. In the second case, \sref{2}{3} and $\sigma_2(H)=\|HQ\|_2$ give the claim.
\end{steps}
\step{1}{2} Equality holds for $H=(1-t)P_\pi+tJ/n$, $t\in[0,1]$.
\pf The entries of $H$ off $\pi$ are $t/n$, and the entries on $\pi$ are $1-t+t/n\ge t/n$. Thus $\delta(H)=t/n$. For $x\perp\one$, $Hx=(1-t)P_\pi x$ and $\|P_\pi x\|_2=\|x\|_2$. Thus $\sigma_2(H)=1-t=1-n\delta(H)$.
\step{1}{3} Part (b) holds.
\pf By (F3), each $H_i$ maps $\one^\perp$ into $\one^\perp$, and on $\one^\perp$ it has norm $\sigma_2(H_i)$. Thus $\|H_{L:l}x\|_2\le\prod_i\sigma_2(H_i)\|x\|_2$ for $x\perp\one$. Then \sref{1}{1} and $1-s\le e^{-s}$ give the second inequality.
\step{1}{4} Part (c) holds.
\pf By (F2), $H_{L:l}\in\Bn$, so $H_{L:l}J=J$ and $H_{L:l}-J/n=H_{L:l}Q$. For unit vectors $e_j,e_k$,
\[
\big|(H_{L:l})_{jk}-\tfrac1n\big|=|e_j\T H_{L:l}Qe_k|\le\|H_{L:l}Q\|_2=\sigma_2(H_{L:l}).
\]
The limit statement then follows from (b).
\step{1}{5} \Assume $H=\SK(\exp W)$ with $\|W\|_\infty\le M$, for the exact projection or for a truncated iteration that ends with a row normalization. \Prove $\delta(H)\ge e^{-4M}/n$.
\begin{steps}
\step{2}{1} \Let $K:=\exp W$. Then $H=\diag(\mu)K\diag(\nu)$ with $\mu,\nu>0$, and $\nu_j=1/\sum_i\mu'_iK_{ij}$ for some $\mu'>0$.
\pf For the exact projection, the column sums of $H$ are one, so this holds with $\mu'=\mu$. For a truncated iteration, the last column normalization sets $\nu$, and $\mu'$ is the row scaling at that moment.
\step{2}{2} $\nu_j/\nu_k\ge e^{-2M}$ for all $j,k$.
\pf Let $s:=\sum_i\mu'_i$. Because $K_{ij}\in[e^{-M},e^M]$, we have $\sum_i\mu'_iK_{ij}\in[e^{-M}s,e^{M}s]$, so $\nu_j\in[e^{-M}/s,e^{M}/s]$ by \sref{2}{1}.
\step{2}{3} $H_{ij}/H_{ik}\ge e^{-4M}$ for all $i,j,k$.
\pf $H_{ij}/H_{ik}=(K_{ij}/K_{ik})(\nu_j/\nu_k)\ge e^{-2M}e^{-2M}$ by \sref{2}{2}.
\qedstep{2}{4}
\pf Row $i$ of $H$ sums to one, so $H_{ik}\ge1/n$ for some $k$. By \sref{2}{3}, $H_{ij}\ge e^{-4M}/n$ for every $j$.
\end{steps}
\step{1}{6} For the exact projection, $\sigma_2(H)\le1-e^{-4M}$.
\pf For the exact projection $H\in\Bn$. Apply \sref{1}{1} and \sref{1}{5}.
\qedstep{1}{7}
\pf Steps \sref{1}{1} and \sref{1}{2} give (a), \sref{1}{3} gives (b), \sref{1}{4} gives (c), and \sref{1}{5} and \sref{1}{6} give (d).
\end{steps}
\end{proof}

\begin{proof}[Proof of Corollary~\ref{cor:fading}]
\begin{steps}
\step{1}{1} $\|H_{L:l}x^\perp_l\|_F\le\rho_{L:l}\|x^\perp_l\|_F$.
\pf Every column of $x^\perp_l$ is in $\one^\perp$. Apply Theorem~\ref{thm:contraction}(b) to each column.
\step{1}{2} $\|H_{L:i+1}Q(\hpost_i)\T u_i\|_F\le\rho_{L:i+1}\|Q(\hpost_i)\T\|_2\|u_i\|_2$.
\pf The matrix $Q(\hpost_i)\T u_i$ has rank one, and its column space is spanned by the vector $Q(\hpost_i)\T\in\one^\perp$. Thus $H_{L:i+1}Q(\hpost_i)\T u_i=(H_{L:i+1}Q(\hpost_i)\T)\,u_i$, and its Frobenius norm is $\|H_{L:i+1}Q(\hpost_i)\T\|_2\|u_i\|_2$. Apply Theorem~\ref{thm:contraction}(b) to the vector $Q(\hpost_i)\T$.
\step{1}{3} The bound \eqref{eq:fading} holds.
\pf Apply the triangle inequality to \eqref{eq:two-channel-unrolled}, and use \sref{1}{1} and \sref{1}{2}.
\step{1}{4} $\|QH_{L:l}\T y\|_F\le\rho_{L:l}\|Qy\|_F$.
\pf By (F3), $QH_{L:l}\T=H_{L:l}\T Q$. Each $H_i\T$ is in $\Bn$ and has $\sigma_2(H_i\T)=\sigma_2(H_i)$. Apply Theorem~\ref{thm:contraction}(b) to $H_{L:l}\T=H_l\T\cdots H_{L-1}\T$ and the columns of $Qy$.
\step{1}{5} If $\rho\in[0,1)$ and $\tau=1/(1-\rho)$, then $\rho^k\le e^{-k/\tau}$.
\pf $\log\rho\le\rho-1=-1/\tau$.
\qedstep{1}{6}
\pf By \sref{1}{3}--\sref{1}{5}.
\end{steps}
\end{proof}

\begin{proof}[Proof of Proposition~\ref{prop:rigidity}]
\begin{steps}
\step{1}{1} Every $A\in\Cn$ satisfies $A\T\one=\one$, and $A$ maps $\one^\perp$ into $\one^\perp$.
\pf By the Cauchy--Schwarz inequality and $\|A\T\|_2=\|A\|_2\le1$,
\[
n=\one\T A\one=\ip{A\T\one}{\one}\le\|A\T\one\|_2\,\|\one\|_2\le\|\one\|_2^2=n .
\]
Thus both inequalities are equalities. Equality in the Cauchy--Schwarz inequality, together with $\|A\T\one\|_2=\|\one\|_2$, gives $A\T\one=\one$. For $x\perp\one$, $\ip{Ax}{\one}=\ip{x}{A\T\one}=0$.
\step{1}{2} $\Cn=\{J/n+UBU\T:\ B\in\R^{(n-1)\times(n-1)},\ \|B\|_2\le1\}$, where $U$ is an orthonormal basis of $\one^\perp$.
\pf Let $A\in\Cn$. By \sref{1}{1}, $A$ preserves $\R\one$ and $\one^\perp$, so $A=J/n+U(U\T AU)U\T$, and $\|U\T AU\|_2\le\|A\|_2\le1$. Conversely, $J/n+UBU\T$ fixes $\one$, and it is block diagonal in the basis $[\one/\sqrt n,\ U]$ with blocks $1$ and $B$. Thus its norm is at most one.
\step{1}{3} For $Z\in\R^{(n-1)\times(n-1)}$, the matrix $J/n+UZU\T$ is in $\On$ if and only if $Z\in O(n-1)$.
\pf Because $U\T\one=0$, we have $(J/n+UZU\T)\T(J/n+UZU\T)=J/n+UZ\T ZU\T$. This matrix equals $I=J/n+UU\T$ if and only if $Z\T Z=I$.
\step{1}{4} An element $A\in\Cn$ is collapse-free if and only if $A\in\On$.
\pf Write $A=J/n+UBU\T$ as in \sref{1}{2}. For $x=Uy\perp\one$, $\|Ax\|_2=\|By\|_2$ and $\|x\|_2=\|y\|_2$. Thus $A$ is collapse-free if and only if $B$ is orthogonal, that is, by \sref{1}{3}, if and only if $A\in\On$.
\step{1}{5} $\Cn=\conv\On$.
\pf The set $\Cn$ is convex, and it contains $\On$. Conversely, let $A=J/n+UBU\T\in\Cn$. By Lemma~\ref{aux:convO}, $B=\sum_s\lambda_sZ_s$ with $Z_s\in O(n-1)$, $\lambda_s\ge0$, and $\sum_s\lambda_s=1$. Thus $A=\sum_s\lambda_s(J/n+UZ_sU\T)$, and each term is in $\On$ by \sref{1}{3}.
\step{1}{6} $\Bn=\Cn\cap\R^{n\times n}_{\ge0}$.
\pf By (F1), $\Bn\subseteq\Cn$. Conversely, if $A\ge0$, $A\one=\one$, and $\|A\|_2\le1$, then $A\T\one=\one$ by \sref{1}{1}, so $A\in\Bn$.
\step{1}{7} $\Bn\cap\On=\Sn$.
\pf Let $A\in\Bn\cap\On$. Then $\|A\|_F^2=\tr(A\T A)=n$. Also $0\le A_{ij}\le1$, so $\sum_{ij}A_{ij}^2\le\sum_{ij}A_{ij}=n$, with equality if and only if every entry is $0$ or $1$. A doubly stochastic matrix with entries in $\{0,1\}$ is a permutation matrix. Conversely, permutation matrices are orthogonal and doubly stochastic.
\step{1}{8} For $H\in\Bn$, $\sigma_2(H)=1$ if and only if $H$ is partly decomposable.
\begin{steps}
\step{2}{1} If $H$ is partly decomposable, then $\sigma_2(H)=1$.
\pf Let $\mathcal R,\mathcal S$ be as in the definition, and let $k:=|\mathcal R|=|\mathcal S|$. Let $x:=\mathbf 1_{\mathcal S}/k-\mathbf 1_{\mathcal S^c}/(n-k)$, so $x\perp\one$. For $i\in\mathcal R$, row $i$ is supported in $\mathcal S$ and sums to one, so $(Hx)_i=1/k$. For $i\notin\mathcal R$, row $i$ is supported in $\mathcal S^c$, so $(Hx)_i=-1/(n-k)$. Thus
\[
\|Hx\|_2^2=\frac{k}{k^2}+\frac{n-k}{(n-k)^2}=\|x\|_2^2 ,
\]
so $\sigma_2(H)\ge1$, and $\sigma_2(H)\le1$ by (F1).
\step{2}{2} If $\sigma_2(H)=1$, then $H$ is partly decomposable.
\begin{steps}
\step{3}{1} There is $x\perp\one$, $x\ne0$, with $H\T Hx=x$.
\pf By compactness, there is a unit vector $x\perp\one$ with $\|Hx\|_2=1$. By (F1), $I-H\T H$ is positive semidefinite, and $\ip{(I-H\T H)x}{x}=0$. Thus $(I-H\T H)x=0$.
\step{3}{2} The matrix $N:=H\T H$ is reducible.
\pf $N$ is symmetric, nonnegative, and doubly stochastic, so its spectral radius is $1$, with eigenvector $\one$. By \sref{3}{1}, the eigenvalue $1$ is not simple. By Theorem~\ref{aux:pf}(a), $N$ is not irreducible.
\step{3}{3} There is a nonempty proper subset $\mathcal S\subset[n]$ such that no row of $H$ has entries in both $\mathcal S$ and $\mathcal S^c$.
\pf By \sref{3}{2} and Theorem~\ref{aux:pf}(b), there is a nonempty proper $\mathcal S$ with $N_{jk}=\sum_iH_{ij}H_{ik}=0$ for $j\in\mathcal S$ and $k\notin\mathcal S$. All terms are nonnegative, so $H_{ij}H_{ik}=0$ for every $i$.
\qedstep{3}{4}
\pf Each row of $H$ sums to one, so by \sref{3}{3} it is supported either in $\mathcal S$ or in $\mathcal S^c$. Let $\mathcal R$ be the set of rows supported in $\mathcal S$. If we sum the columns in $\mathcal S$, we get
\[
|\mathcal S|=\sum_{j\in\mathcal S}\sum_iH_{ij}=\sum_{i\in\mathcal R}\sum_{j\in\mathcal S}H_{ij}=|\mathcal R| .
\]
Thus $0<|\mathcal R|=|\mathcal S|<n$. By construction $H_{ij}=0$ if $i\in\mathcal R$ and $j\notin\mathcal S$, or if $i\notin\mathcal R$ and $j\in\mathcal S$.
\end{steps}
\qedstep{2}{3}
\pf Steps \sref{2}{1} and \sref{2}{2} give the equivalence. A positive matrix is not partly decomposable, so $\sigma_2<1$ on $\Bno$.
\end{steps}
\step{1}{9} \Assume $H\in\Bn$, $\sigma_n(H)\ge1-\vartheta$, and $4n\vartheta<1$. \Prove $\eps_\pi(H)\le4n\vartheta$ for some $\pi$.
\begin{steps}
\step{2}{1} $\sum_{ij}H_{ij}(1-H_{ij})\le2n\vartheta$.
\pf We have $\|H\|_F^2=\sum_k\sigma_k(H)^2\ge n(1-\vartheta)^2\ge n(1-2\vartheta)$. Also $\|H\|_F^2=\sum_{ij}H_{ij}-\sum_{ij}H_{ij}(1-H_{ij})=n-\sum_{ij}H_{ij}(1-H_{ij})$.
\step{2}{2} The entries with $H_{ij}\le\frac12$ have a sum of at most $4n\vartheta<1$.
\pf These entries satisfy $H_{ij}(1-H_{ij})\ge H_{ij}/2$. Apply \sref{2}{1}.
\step{2}{3} Each row and each column of $H$ has exactly one entry larger than $\frac12$.
\pf If a row has no such entry, then all its entries are counted in \sref{2}{2}, and their sum is $1>4n\vartheta$. This is a contradiction. Two entries larger than $\frac12$ cannot be in one row, because the row sums to one. The same arguments apply to columns.
\step{2}{4} \Let $\pi(i)$ be the column of the large entry of row $i$. Then $\pi$ is a permutation.
\pf If $\pi(i)=\pi(i')$ with $i\ne i'$, then column $\pi(i)$ has two entries larger than $\frac12$. This contradicts \sref{2}{3}.
\qedstep{2}{5}
\pf By \sref{2}{3}, every entry off $\pi$ is at most $\frac12$, so \sref{2}{2} counts it. Thus $\eps_\pi(H)\le4n\vartheta$.
\end{steps}
\qedstep{1}{10}
\pf Steps \sref{1}{1}, \sref{1}{4}, \sref{1}{5}, and \sref{1}{6} give (a); \sref{1}{7} gives (b); \sref{1}{8} gives (c); and \sref{1}{9} gives (d).
\end{steps}
\end{proof}

\subsection{Proofs for Section~\ref{sec:geometry}}

\begin{proof}[Proof of Lemma~\ref{lem:chart}]
\begin{steps}
\step{1}{1} $\Phi(W)=\Phi(W')$ if and only if $W-W'\in\gauge$.
\pf If $W'=W+a\one\T+\one b\T$, then $\exp W'=\diag(e^a)\exp(W)\diag(e^b)$. By the uniqueness in Theorem~\ref{aux:sinkhorn}(a), this matrix has the same doubly stochastic scaling as $\exp W$. Conversely, if $\Phi(W)=\Phi(W')$, then $\diag(e^\alpha)e^{W}\diag(e^\beta)=\diag(e^{\alpha'})e^{W'}\diag(e^{\beta'})$. The entrywise logarithm gives $W-W'=(\alpha'-\alpha)\one\T+\one(\beta'-\beta)\T\in\gauge$.
\step{1}{2} $\Phi$ is surjective onto $\Bno$.
\pf For $H\in\Bno$, the matrix $\exp(\log H)=H$ is doubly stochastic, so $\Phi(\log H)=H$.
\step{1}{3} $\Phi$ is smooth, and we can choose the scaling vectors $\alpha(W),\beta(W)$ smoothly.
\begin{steps}
\step{2}{1} \Let $M(W,\alpha,\beta):=\diag(e^\alpha)e^{W}\diag(e^\beta)$ and $F(W,\alpha,\beta):=(M\one-\one,\ M\T\one-\one)$. At a zero of $F$, the Jacobian of $F$ with respect to $(\alpha,\beta)$ is
$\begin{psmallmatrix}I&H\\H\T&I\end{psmallmatrix}$, where $H=M$.
\pf $\partial(M\one)_i/\partial\alpha_k=(M\one)_i$ if $k=i$ and $0$ otherwise, and $\partial(M\one)_i/\partial\beta_j=M_{ij}$. The column sums are similar. At a zero of $F$, $M\one=M\T\one=\one$.
\step{2}{2} The kernel of this Jacobian is the line $\{(s\one,-s\one):s\in\R\}$.
\pf A pair $(a,b)$ in the kernel satisfies $a=-Hb$ and $b=-H\T a=H\T Hb$. By Proposition~\ref{prop:rigidity}(c), $\sigma_2(H)<1$, so the eigenvalue $1$ of $H\T H$ is simple, with eigenvector $\one$. Thus $b=s\one$ and $a=-s\one$.
\step{2}{3} On the slice $\{\sum_j\beta_j=0\}$, the Jacobian is a bijection onto the hyperplane $E:=\{(r,c)\in\R^{2n}:\ \sum_ir_i=\sum_jc_j\}$.
\pf The total of the row sums of $M$ equals the total of its column sums, so $F$ takes values in $E$, and the Jacobian maps into $E$. The kernel line of \sref{2}{2} meets the slice only in $0$. Thus the Jacobian is injective on the slice, which has dimension $2n-1=\dim E$.
\qedstep{2}{4}
\pf Apply the implicit function theorem to $F:\R^{n\times n}\times(\R^n\times\{\sum_j\beta_j=0\})\to E$, with \sref{2}{3}. This gives smooth $\alpha(W),\beta(W)$, and thus $\Phi(W)=M(W,\alpha(W),\beta(W))$ is smooth.
\end{steps}
\step{1}{4} The Frobenius-orthogonal projection onto $\Tsp$ is $\xi\mapsto Q\xi Q$.
\pf The matrix $Q\xi Q$ has zero row sums and zero column sums. Also, $\xi-Q\xi Q=\frac1nJ\xi+\frac1n\xi J-\frac1{n^2}J\xi J$, and each of the three terms is in $\gauge=\Tsp^\perp$.
\step{1}{5} $\Phi|_\Tsp$ is a diffeomorphism onto $\Bno$, with inverse $H\mapsto Q(\log H)Q$.
\pf For $W\in\Tsp$, $\log\Phi(W)=W+\alpha\one\T+\one\beta\T$. By \sref{1}{4}, $Q(\log\Phi(W))Q=QWQ=W$. For $H\in\Bno$, $\log H-Q(\log H)Q\in\gauge$ by \sref{1}{4}. Thus \sref{1}{1} and \sref{1}{2} give $\Phi(Q(\log H)Q)=\Phi(\log H)=H$. Both maps are smooth by \sref{1}{3}.
\step{1}{6} If $W_k\in\Tsp$ and $\Phi(W_k)\to P_\pi$, then $\|W_k\|_\infty\to\infty$.
\pf Assume the opposite. Then a subsequence satisfies $\|W_k\|_\infty\le M$ for some $M$, and Theorem~\ref{thm:contraction}(d) gives $\delta(\Phi(W_k))\ge e^{-4M}/n$ on this subsequence. But $P_\pi$ has zero entries because $n\ge2$, so $\delta(\Phi(W_k))\to0$. This is a contradiction.
\qedstep{1}{7}
\pf By \sref{1}{1}--\sref{1}{3}, \sref{1}{5}, and \sref{1}{6}.
\end{steps}
\end{proof}

\begin{proof}[Proof of Lemma~\ref{lem:differential}]
\begin{steps}
\step{1}{1} $\dd\Phi_W(V)=\Pfr_H(H\odot V)$.
\begin{steps}
\step{2}{1} \Let $H(t):=\Phi(W+tV)$. Then $\dot H(0)\in H\odot V+H\odot\gauge$.
\pf By Lemma~\ref{lem:chart} (step \xsref{1}{3} of its proof), $\log H(t)=W+tV+\alpha(t)\one\T+\one\beta(t)\T$ with smooth $\alpha,\beta$. Differentiate at $t=0$: $\dot H(0)=H\odot(V+\dot\alpha\one\T+\one\dot\beta\T)$.
\step{2}{2} $\dot H(0)\in\Tsp$.
\pf $H(t)\in\Bn$ for all $t$, so the row sums and the column sums of $\dot H(0)$ are zero.
\qedstep{2}{3}
\pf By Lemma~\ref{lem:gauge-complement}(b), the only point of $H\odot V+H\odot\gauge$ in $\Tsp$ is $\Pfr_H(H\odot V)$.
\end{steps}
\step{1}{2} $g_H(\zeta,H\odot\theta)=\ip{\zeta}{\theta}$ for all $\zeta,\theta\in\R^{n\times n}$.
\pf $\sum_{ij}\zeta_{ij}H_{ij}\theta_{ij}/H_{ij}=\ip{\zeta}{\theta}$.
\step{1}{3} $A_H$ is self-adjoint.
\pf Let $\xi,\zeta\in\R^{n\times n}$, and write $\zeta=\zeta_\Tsp+\zeta_\gauge$ with $\zeta_\Tsp\in\Tsp$ and $\zeta_\gauge\in\gauge$. We use three facts: $A_H\xi\in\Tsp$, step \sref{1}{2}, and the $g_H$-self-adjointness of the $g_H$-orthogonal projection $\Pfr_H$. Then
\begin{align*}
\ip{A_H\xi}{\zeta}&=\ip{A_H\xi}{\zeta_\Tsp}
=g_H\big(\Pfr_H(H\odot\xi),\,H\odot\zeta_\Tsp\big)
=g_H\big(H\odot\xi,\,\Pfr_H(H\odot\zeta_\Tsp)\big)\\
&=\ip{\xi}{\Pfr_H(H\odot\zeta_\Tsp)}
=\ip{\xi}{\Pfr_H(H\odot\zeta)}=\ip{\xi}{A_H\zeta}.
\end{align*}
The fifth equality holds because $H\odot\zeta_\gauge\in\ker\Pfr_H$ (Lemma~\ref{lem:gauge-complement}(c)).
\step{1}{4} $A_H$ is positive semidefinite, with kernel $\gauge$.
\pf By \sref{1}{2} and because $\Pfr_H$ is a $g_H$-orthogonal projection,
\[
\ip{A_H\xi}{\xi}=g_H\big(\Pfr_H(H\odot\xi),H\odot\xi\big)=g_H\big(\Pfr_H(H\odot\xi),\Pfr_H(H\odot\xi)\big)\ge0 .
\]
Equality holds if and only if $H\odot\xi\in\ker\Pfr_H=H\odot\gauge$, that is, $\xi\in\gauge$.
\step{1}{5} The restriction $A_H|_\Tsp$ is a positive definite bijection of $\Tsp$, and $g_H(\xi,\zeta)=\ip{A_H|_\Tsp^{-1}\xi}{\zeta}$ for $\xi,\zeta\in\Tsp$.
\pf $A_H$ maps into $\Tsp$, and $\Tsp\cap\gauge=\{0\}$. Thus $A_H|_\Tsp$ is injective by \sref{1}{4}, hence bijective and positive definite. Let $\xi\in\Tsp$ and $\theta:=A_H|_\Tsp^{-1}\xi$, so $\Pfr_H(H\odot\theta)=\xi$. By Lemma~\ref{lem:gauge-complement}(b), $H\odot\theta=\xi+H\odot s$ with $s\in\gauge$, that is, $\theta=\xi\oslash H+s$. For $\zeta\in\Tsp$, $\ip{\theta}{\zeta}=\ip{\xi\oslash H}{\zeta}+\ip{s}{\zeta}=g_H(\xi,\zeta)+0$.
\step{1}{6} $\nabla_W(f\circ\Phi)(W)=A_H(\nabla f(H))=\gradFR f(H)$.
\pf By the chain rule, \sref{1}{1}, and \sref{1}{3}, $\nabla_W(f\circ\Phi)=(\dd\Phi_W)^*\nabla f(H)=A_H\nabla f(H)$. By \eqref{eq:FRgrad}, $A_H\nabla f(H)=\Pfr_H(H\odot\nabla f(H))=\gradFR f(H)$.
\qedstep{1}{7}
\pf Step \sref{1}{1} gives (a), \sref{1}{3}--\sref{1}{5} give (b), and \sref{1}{6} gives (c).
\end{steps}
\end{proof}

\begin{proof}[Proof of Theorem~\ref{thm:induced}]
\begin{steps}
\step{1}{1} \Let $B:=A_H|_\Tsp$, the positive definite endomorphism of $\Tsp$ of Lemma~\ref{lem:differential}(b). Then $A_H\xi=B(Q\xi Q)$ for every $\xi$.
\pf $A_H$ is zero on $\gauge=\Tsp^\perp$, and $\xi\mapsto Q\xi Q$ is the Frobenius projection onto $\Tsp$ (Lemma~\ref{lem:chart}, step \xsref{1}{4}).
\step{1}{2} Under the logit gradient flow, $\dot H=-A_H^2\nabla f(H)=-A_H\gradFR f(H)$.
\pf By Lemma~\ref{lem:differential}(c), $\dot W=-A_H\nabla f(H)$. By Lemma~\ref{lem:differential}(a), $\dot H=\dd\Phi_W(\dot W)=A_H\dot W$.
\step{1}{3} The logit gradient flow is the gradient flow of $f$ for $\mathcal G$.
\pf The $\mathcal G$-gradient of $f$ at $H$ is the $\xi\in\Tsp$ with $\mathcal G_H(\xi,\zeta)=\ip{\nabla f(H)}{\zeta}$ for all $\zeta\in\Tsp$. Because $\mathcal G_H(\xi,\zeta)=\ip{B^{-2}\xi}{\zeta}$, this condition says $B^{-2}\xi=Q\nabla f(H)Q$, that is, $\xi=B^2Q\nabla f(H)Q=A_H^2\nabla f(H)$ by \sref{1}{1}. Thus \sref{1}{2} says $\dot H=-\xi$. The tensor $B$ depends smoothly on $H$ and is positive definite, so $\mathcal G$ is a Riemannian metric on $\Bno$.
\step{1}{4} Under the straight-through flow, $\dot H=-\gradFR f(H)$, and this flow is the entropic mirror flow.
\pf By Lemma~\ref{lem:differential}(a) and (c), $\dot H=A_H(-\nabla f(H))=-\gradFR f(H)$. The entropic mirror flow on $\Bn$ has the mirror map $\log H$, and it reads $\frac{\dd}{\dd t}\log H=-\nabla f(H)+s$, where $s\in\gauge$ keeps $H$ in $\Bn$. Thus $\dot H=H\odot(-\nabla f(H)+s)\in\Tsp$. By Lemma~\ref{lem:gauge-complement}(b), the only such vector is $-\Pfr_H(H\odot\nabla f(H))=-\gradFR f(H)$.
\step{1}{5} $\SK(D_1KD_2)=\SK(K)$ for all positive diagonal $D_1,D_2$ and all $K>0$.
\pf If $D_1'KD_2'$ is doubly stochastic, then so is $(D_1'D_1^{-1})(D_1KD_2)(D_2^{-1}D_2')$. Apply the uniqueness in Theorem~\ref{aux:sinkhorn}(a).
\step{1}{6} $\Phi(W^{(k)}-\eta\nabla f(H^{(k)}))=\SK(H^{(k)}\odot e^{-\eta\nabla f(H^{(k)})})$.
\pf Write $H^{(k)}=\diag(e^\alpha)\exp(W^{(k)})\diag(e^\beta)$ and $\Gamma:=\nabla f(H^{(k)})$. Then $\exp(W^{(k)}-\eta\Gamma)=\diag(e^{-\alpha})(H^{(k)}\odot e^{-\eta\Gamma})\diag(e^{-\beta})$. Apply \sref{1}{5}.
\step{1}{7} The minimizer of $\eta\ip{\Gamma}{H}+\KL(H\,\|\,H^{(k)})$ over $\Bn$ is $\SK(K)$, where $K:=H^{(k)}\odot e^{-\eta\Gamma}$.
\begin{steps}
\step{2}{1} On $\Bn$, $\eta\ip{\Gamma}{H}+\KL(H\,\|\,H^{(k)})=\KL(H\,\|\,K)+\text{const}$.
\pf Both sides equal $\sum_{ij}H_{ij}\log(H_{ij}/K_{ij})$ plus a constant, because $\sum_{ij}H_{ij}=n$ on $\Bn$.
\step{2}{2} The function $H\mapsto\KL(H\,\|\,K)$ has a unique minimizer on $\Bn$.
\pf It is continuous and strictly convex on the compact convex set $\Bn$.
\step{2}{3} The minimizer is in $\Bno$.
\pf Let $H\in\Bn$ with $H_{ij}=0$. Choose a permutation $\pi'$ with $\pi'(i)=j$, and move along $(1-s)H+sP_{\pi'}$, $s\in[0,1]$. Each cell that is zero in $H$ and one in $P_{\pi'}$ changes the objective by $s\log s+O(s)$. Each other cell changes it by $O(s)$. Thus the objective decreases for small $s>0$, and $H$ is not a minimizer.
\qedstep{2}{4}
\pf At the interior minimizer, the Lagrange conditions for the marginal constraints give $\log H_{ij}-\log K_{ij}=a_i+b_j$, that is, $H=\diag(e^a)K\diag(e^b)$ \citep{csiszar1975idivergence}. This matrix is doubly stochastic, so it is $\SK(K)$.
\end{steps}
\step{1}{8} $R_{H^{(k)}}(-\eta\gradFR f(H^{(k)}))=\SK(K)$.
\pf By Lemma~\ref{lem:fit-normal}, $\gradFR f(H^{(k)})=H^{(k)}\odot(\Gamma-m)$ with $m=a\one\T+\one b\T\in\gauge$. Thus
\[
H^{(k)}\odot\exp\big(-\eta\gradFR f(H^{(k)})\oslash H^{(k)}\big)=H^{(k)}\odot e^{-\eta\Gamma}\odot e^{\eta m}=\diag(e^{\eta a})\,K\,\diag(e^{\eta b}).
\]
Apply \sref{1}{5}.
\step{1}{9} $R$ is a first-order retraction on $\Bno$.
\pf $R_H(t\xi)=\Phi(\log H+t\,\xi\oslash H)$ is smooth in $(H,\xi)$ by Lemma~\ref{lem:chart}, and it equals $\Phi(\log H)=H$ at $t=0$. By Lemma~\ref{lem:differential}(a), its derivative at $t=0$ is $\Pfr_H(H\odot(\xi\oslash H))=\Pfr_H(\xi)=\xi$ for $\xi\in\Tsp$.
\qedstep{1}{10}
\pf Steps \sref{1}{2} and \sref{1}{3} give (a), \sref{1}{4} gives (b), and \sref{1}{6}--\sref{1}{9} give (c).
\end{steps}
\end{proof}

\subsection{Proofs for Section~\ref{sec:viscosity}}

\begin{proof}[Proof of Lemma~\ref{lem:additive}]
\begin{steps}
\step{1}{1} \Let $z:=(a_1,\dots,a_n,b_1,\dots,b_{n-1})\in\R^{2n-1}$, with the gauge $b_n:=0$. \Let $X\in\R^{n^2\times(2n-1)}$ be the design matrix: the row of the cell $(i,j)$ has the entry $1$ in the columns of $a_i$ and $b_j$ (only in the column of $a_i$ if $j=n$), and $0$ elsewhere. \Let $y:=\mathrm{vec}\,C$ and $D:=\diag(\omega)$. Then $(Xz)_{ij}=a_i+b_j$, and the problem is $\min_z\|D^{1/2}(y-Xz)\|_2^2$.
\pf By the definitions. The gauge $b_n=0$ does not change the set of fitted values, because $(a+s\one,b-s\one)$ and $(a,b)$ have the same fitted values.
\step{1}{2} $X$ has full column rank. Thus the minimizer $z^\star=(X\T DX)^{-1}X\T Dy$ is unique, and the fitted values $m=Xz^\star$ are unique.
\pf If $Xz=0$, then $a_i=-b_j$ for all $i,j$, so all $b_j$ are equal. Thus $b=0$ (because $b_n=0$), and $a=0$.
\step{1}{3} $z^\star=\sum_{S}\lambda_S\,z_S$, where $S$ ranges over the $(2n-1)$-element sets of cells with $\det X_S\ne0$, $z_S$ solves $X_Sz=y_S$, and
\begin{equation}\label{eq:jacobi}
\lambda_S=\frac{\omega_S\det(X_S)^2}{\sum_{S'}\omega_{S'}\det(X_{S'})^2}\ge0,\qquad\sum_S\lambda_S=1 .
\end{equation}
\begin{steps}
\step{2}{1} $z^\star_r=\det(X\T DX^{(r)})/\det(X\T DX)$, where $X^{(r)}$ is $X$ with its $r$-th column replaced by $y$.
\pf By Cramer's rule (Theorem~\ref{aux:cb}), $z^\star_r=\det(N_r)/\det(X\T DX)$, where $N_r$ is $X\T DX$ with its $r$-th column replaced by $X\T Dy$. This column is the $r$-th column of $X\T DX^{(r)}$, and the other columns agree, so $N_r=X\T DX^{(r)}$.
\step{2}{2} $\det(X\T DX^{(r)})=\sum_S\omega_S\det(X_S)\det(X^{(r)}_S)$ and $\det(X\T DX)=\sum_S\omega_S\det(X_S)^2$.
\pf The Cauchy--Binet formula (Theorem~\ref{aux:cb}).
\qedstep{2}{3}
\pf For $S$ with $\det X_S\ne0$, Cramer's rule for $X_Sz=y_S$ gives $(z_S)_r=\det(X^{(r)}_S)/\det(X_S)$. Insert this into \sref{2}{2}, and divide as in \sref{2}{1}. (A representation of this kind is classical; see \citet{bental1990geometric}.)
\end{steps}
\step{1}{4} $\det X_S\ne0$ if and only if the cells of $S$ form a spanning tree of the complete bipartite graph $K_{n,n}$.
\pf View each cell $(i,j)$ as the edge between the row vertex $i$ and the column vertex $j$ of $K_{n,n}$. This graph is connected and has $2n$ vertices. If we multiply the $b$-columns of $X$ by $-1$, then $X\T$ becomes the signed incidence matrix of $K_{n,n}$ (edges oriented from rows to columns), without the row of the vertex $b_n$. This change multiplies each $\det X_S$ by $\pm1$. Apply Theorem~\ref{aux:tree}.
\step{1}{5} For every spanning tree $S$ and every cell $(i,j)$, $|(Xz_S)_{ij}|\le(2n-1)\|C\|_\infty$.
\pf The unique path in $S$ from the row vertex $i$ to the column vertex $j$ alternates between rows and columns: $i=i_0,\ j_1,\ i_1,\ j_2,\ \dots,\ i_k,\ j_{k+1}=j$. Its edges are the cells $(i_r,j_{r+1})$ for $0\le r\le k$ and $(i_r,j_r)$ for $1\le r\le k$, and all of them are in $S$. Because $z_S$ interpolates $C$ on $S$, a telescoping sum gives
\begin{equation}\label{eq:path-sum}
a_i+b_j=\sum_{r=0}^{k}\big(a_{i_r}+b_{j_{r+1}}\big)-\sum_{r=1}^{k}\big(a_{i_r}+b_{j_r}\big)
=\sum_{r=0}^{k}C_{i_rj_{r+1}}-\sum_{r=1}^{k}C_{i_rj_r}.
\end{equation}
This is a signed sum of $2k+1$ entries of $C$. A path in a tree with $2n$ vertices has at most $2n-1$ edges, so $2k+1\le2n-1$.
\step{1}{6} $\|m\|_\infty\le(2n-1)\|C\|_\infty$.
\pf By \sref{1}{3} and \sref{1}{4}, $m=Xz^\star=\sum_S\lambda_SXz_S$ is a convex combination over spanning trees. Apply \sref{1}{5}.
\step{1}{7} The constant $2n-1$ cannot be made smaller.
\pf Let $S_0$ be the Hamiltonian path $r_1,c_1,r_2,c_2,\dots,r_n,c_n$ through the row vertices $r_k$ and the column vertices $c_k$. Its cells are $(r_k,c_k)$ for $1\le k\le n$ and $(r_k,c_{k-1})$ for $2\le k\le n$; it is a spanning tree with $2n-1$ cells. Let $C=+1$ on the cells $(r_k,c_k)$, $C=-1$ on the cells $(r_k,c_{k-1})$, and $C=0$ elsewhere, so $\|C\|_\infty=1$. Let $\omega=1$ on $S_0$ and $\omega=\nu$ elsewhere. Every spanning tree other than $S_0$ has a cell of weight $\nu$, and $\det(X_S)^2=1$ for every tree. Thus $\lambda_{S_0}\to1$ in \eqref{eq:jacobi} as $\nu\to0$, and $m\to Xz_{S_0}$. By \eqref{eq:path-sum}, the fitted value of $z_{S_0}$ at the cell $(r_1,c_n)$ is the alternating sum along the whole path, $n\cdot(+1)-(n-1)\cdot(-1)=2n-1$.
\qedstep{1}{8}
\pf Step \sref{1}{2} gives existence and uniqueness, \sref{1}{6} gives \eqref{eq:additive}, and \sref{1}{7} gives sharpness.
\end{steps}
\end{proof}

\begin{proof}[Proof of Lemma~\ref{lem:small-grad}]
\begin{steps}
\step{1}{1} $\gradFR f(H)=H\odot(\nabla f(H)-m)$ with $\|m\|_\infty\le(2n-1)G$.
\pf By \eqref{eq:FRgrad} and Lemma~\ref{lem:fit-normal}, $m$ is the weighted additive fit of $\nabla f(H)$ with weights $H$. Apply Lemma~\ref{lem:additive} with $\omega=H$ and $C=\nabla f(H)$.
\step{1}{2} $|\gradFR f(H)_{ij}|\le2nG\,H_{ij}$ for all $(i,j)$.
\pf By \sref{1}{1}, $|\gradFR f(H)_{ij}|\le H_{ij}(|\partial_{ij}f(H)|+|m_{ij}|)\le H_{ij}(G+(2n-1)G)$.
\step{1}{3} For every $\pi$, $\|\gradFR f(H)\|_\infty\le2n(n-1)G\,\eps_\pi(H)$.
\begin{steps}
\step{2}{1} The bound holds for each cell off $\pi$.
\pf For $j\ne\pi(i)$, $H_{ij}\le\eps_\pi(H)$. Apply \sref{1}{2}, and use $1\le n-1$.
\step{2}{2} The bound holds for each cell on $\pi$.
\pf $\gradFR f(H)\in\Tsp$ has zero row sums. By \sref{1}{2},
\[
\big|\gradFR f(H)_{i\pi(i)}\big|=\Big|\sum_{j\ne\pi(i)}\gradFR f(H)_{ij}\Big|\le2nG\sum_{j\ne\pi(i)}H_{ij}\le2n(n-1)G\,\eps_\pi(H).
\]
\qedstep{2}{3}
\pf By \sref{2}{1} and \sref{2}{2}.
\end{steps}
\step{1}{4} The bound \eqref{eq:flat} holds.
\pf Minimize \sref{1}{3} over $\pi$, and apply the result at $H=\Phi(W)$. By Lemma~\ref{lem:differential}(c), the left-hand side of \eqref{eq:flat} is $\|\gradFR f(\Phi(W))\|_\infty$.
\qedstep{1}{5}
\pf By \sref{1}{2}--\sref{1}{4}.
\end{steps}
\end{proof}

\begin{proof}[Proof of Theorem~\ref{thm:viscosity}]
\begin{steps}
\step{1}{1} $\frac{\dd}{\dd t}\log H_{ij}=\dot W_{ij}+\dot\alpha_i+\dot\beta_j$, where $\alpha(t),\beta(t)$ are the smooth scaling vectors of Lemma~\ref{lem:chart}.
\pf $H(t)=\diag(e^{\alpha(t)})\exp(W(t))\diag(e^{\beta(t)})$, so $\log H_{ij}=W_{ij}+\alpha_i+\beta_j$.
\step{1}{2} $|\dot\alpha_i+\dot\beta_j|\le(2n-1)\|\dot W\|_\infty$ for all $(i,j)$.
\pf Differentiate the constraints $\sum_jH_{ij}=1$ and $\sum_iH_{ij}=1$, and use \sref{1}{1}:
\[
\textstyle\sum_jH_{ij}(\dot W_{ij}+\dot\alpha_i+\dot\beta_j)=0\ \ (i\in[n]),\qquad
\sum_iH_{ij}(\dot W_{ij}+\dot\alpha_i+\dot\beta_j)=0\ \ (j\in[n]).
\]
By Lemma~\ref{lem:fit-normal} with $C=-\dot W$, the pair $(\dot\alpha,\dot\beta)$ gives the weighted additive fit of $-\dot W$ with weights $H$. Apply Lemma~\ref{lem:additive}.
\step{1}{3} $\|\dot W\|_\infty\le2n(n-1)G\,\eps(H)$.
\pf By Lemma~\ref{lem:differential}(c), $\dot W=-\gradFR f(H)$. Apply Lemma~\ref{lem:small-grad}.
\step{1}{4} The bound \eqref{eq:viscosity} holds.
\pf By \sref{1}{1}--\sref{1}{3}, $|\frac{\dd}{\dd t}\log H_{ij}|\le2n\|\dot W\|_\infty\le4n^2(n-1)G\,\eps(H)\le4n^3G\,\eps(H)$.
\step{1}{5} \Let $\phi(t):=\eps(H(t))$ and $c:=4n^3G$. Then $\phi$ is positive and continuous, and for $0\le t<t+h<T$,
\begin{equation}\label{eq:phi-integral}
\phi(t)\,e^{-c\int_t^{t+h}\phi}\ \le\ \phi(t+h)\ \le\ \phi(t)\,e^{c\int_t^{t+h}\phi}.
\end{equation}
\pf $H(t)>0$, so $\phi>0$; and $\phi$ is continuous because $\eps$ is a minimum of maxima of entries. Integrate \eqref{eq:viscosity} over $[t,t+h]$:
\[
H_{ij}(t)\,e^{-c\int_t^{t+h}\phi}\le H_{ij}(t+h)\le H_{ij}(t)\,e^{c\int_t^{t+h}\phi}\quad\text{for every cell}.
\]
The two factors are the same for all cells, so Lemma~\ref{lem:eps-mono} applies.
\step{1}{6} $D_+\phi(t):=\liminf_{h\downarrow0}\frac{\phi(t+h)-\phi(t)}{h}\ge-c\,\phi(t)^2$ and $D^+\phi(t)\le c\,\phi(t)^2$.
\pf By continuity, $\int_t^{t+h}\phi=h\phi(t)+o(h)$ as $h\downarrow0$. Insert this into \eqref{eq:phi-integral}.
\step{1}{7} The approach bound \eqref{eq:approach} holds.
\pf Let $\psi:=1/\phi$. Then $D^+\psi(t)=-D_+\phi(t)/\phi(t)^2\le c$ by \sref{1}{6}. Lemma~\ref{lem:dini} gives $1/\phi(t)\le1/\phi(0)+ct$.
\step{1}{8} The escape bound \eqref{eq:escape} holds.
\pf Let $\chi:=-1/\phi$. Then $D^+\chi(t)=D^+\phi(t)/\phi(t)^2\le c$ by \sref{1}{6}. Lemma~\ref{lem:dini} gives $1/\phi(t)\ge1/\phi(0)-ct$, which is positive for $t<1/(c\phi(0))$.
\qedstep{1}{9}
\pf Steps \sref{1}{4}, \sref{1}{7}, and \sref{1}{8} give the three bounds. If $\phi(t)\le\eps$, then \eqref{eq:approach} gives $t\ge(1/\eps-1/\eps_0)/c$. If $\phi(t)\ge2\eps_0$, then \eqref{eq:escape} gives $t\ge1/(2c\eps_0)$.
\end{steps}
\end{proof}

\begin{proof}[Proof of Theorem~\ref{thm:fr-rate}]
\begin{steps}
\step{1}{1} $|\frac{\dd}{\dd t}\log H_{ij}|\le2nG$, and \eqref{eq:fr-bounds} holds.
\pf By Lemma~\ref{lem:small-grad} (step \xsref{1}{1} of its proof), $\dot H=-H\odot(\nabla f(H)-m)$ with $\|m\|_\infty\le(2n-1)G$. Thus $|\frac{\dd}{\dd t}\log H_{ij}|=|\partial_{ij}f-m_{ij}|\le2nG$. Integrate, and apply Lemma~\ref{lem:eps-mono} with the factors $e^{\mp2nGt}$.
\step{1}{2} For $f_\pi$, $H(t)$ is the doubly stochastic scaling of $K^{(t)}:=\exp(W_0)\odot\big(P_\pi+e^{-t}(J-P_\pi)\big)$.
\pf $\nabla f_\pi=-P_\pi$. By Theorem~\ref{thm:induced}(b), $\Phi(W_0+tP_\pi)$ solves the Fisher--Rao flow, and the solution is unique, so $H(t)=\Phi(W_0+tP_\pi)$. Also $\exp(W_0+tP_\pi)=e^tK^{(t)}$, and a positive multiple does not change the scaling.
\step{1}{3} \Let $H(t)=\diag(\mu)K^{(t)}\diag(\nu)$ with $\mu,\nu>0$. Then $\mu_i\nu_{\pi(i)}\le e^{\|W_0\|_\infty}$ for all $i$.
\pf $H_{i\pi(i)}=\mu_ie^{(W_0)_{i\pi(i)}}\nu_{\pi(i)}\le1$.
\step{1}{4} \Let $t\ge0$ and $\bar\eps:=\eps_\pi(H(t))$. There are rows $i_0,\dots,i_n$ and columns $j_0,\dots,j_n$ with $j_r\ne\pi(i_r)$, $H_{i_rj_r}\ge\bar\eps/(n-1)^r$, and $\pi(i_{r+1})=j_r$.
\begin{steps}
\step{2}{1} The claim holds for $r=0$.
\pf Choose a cell $(i_0,j_0)$ off $\pi$ where $\bar\eps$ is attained.
\step{2}{2} \Assume the claim for $r$. \Prove the claim for $r+1$.
\pf Let $i_{r+1}:=\pi^{-1}(j_r)$. Column $j_r$ sums to one and contains the entry $H_{i_rj_r}\ge\bar\eps/(n-1)^r$ off $\pi$, so $H_{i_{r+1}j_r}\le1-\bar\eps/(n-1)^r$. Row $i_{r+1}$ sums to one, so its $n-1$ entries off $\pi$ have a sum of at least $\bar\eps/(n-1)^r$. One of them, at a column $j_{r+1}\ne\pi(i_{r+1})$, is at least $\bar\eps/(n-1)^{r+1}$.
\qedstep{2}{3}
\pf Induction on $r$.
\end{steps}
\step{1}{5} There are $a<b\le n$ with $i_a=i_b$. \Let $k:=b-a$. Then $k\ge2$, and $H_{i_rj_r}\ge\bar\eps/(n-1)^{n-1}$ for $a\le r<b$.
\pf There are $n+1$ rows $i_0,\dots,i_n$ but only $n$ distinct rows. Also $i_{r+1}\ne i_r$, because $\pi(i_{r+1})=j_r\ne\pi(i_r)$; thus $k\ge2$. For $r<b\le n$, $(n-1)^r\le(n-1)^{n-1}$.
\step{1}{6} $\prod_{r=a}^{b-1}H_{i_rj_r}\le\big(e^{2\|W_0\|_\infty}e^{-t}\big)^{k}$.
\pf Every cell $(i_r,j_r)$ is off $\pi$, so $H_{i_rj_r}=\mu_{i_r}e^{(W_0)_{i_rj_r}}e^{-t}\nu_{j_r}$ with $j_r=\pi(i_{r+1})$. Because $i_b=i_a$, the multisets $\{\pi(i_{r+1})\}_{a\le r<b}$ and $\{\pi(i_r)\}_{a\le r<b}$ are equal. Thus $\prod_r\mu_{i_r}\nu_{j_r}=\prod_r\mu_{i_r}\nu_{\pi(i_r)}\le e^{k\|W_0\|_\infty}$ by \sref{1}{3}. Also $\prod_re^{(W_0)_{i_rj_r}}\le e^{k\|W_0\|_\infty}$.
\step{1}{7} The bound \eqref{eq:fr-attained} holds.
\pf By \sref{1}{5} and \sref{1}{6}, $(\bar\eps/(n-1)^{n-1})^k\le(e^{2\|W_0\|_\infty}e^{-t})^k$. Take the $k$-th root, and use $\eps(H(t))\le\bar\eps$.
\qedstep{1}{8}
\pf Steps \sref{1}{1} and \sref{1}{7} give the first two claims. For the logit gradient flow, apply Theorem~\ref{thm:viscosity} with $G=\|P_\pi\|_\infty=1$.
\end{steps}
\end{proof}

\begin{proof}[Proof of Proposition~\ref{prop:n2}]
\begin{steps}
\step{1}{1} $\Phi(W)=\begin{psmallmatrix}p&q\\q&p\end{psmallmatrix}$ with $\log\frac pq=\frac w2$, and $\Tsp=\R E$ with $E:=\begin{psmallmatrix}1&-1\\-1&1\end{psmallmatrix}$.
\pf Every $2\times2$ doubly stochastic matrix has this form. The cross-ratio $H_{11}H_{22}/(H_{12}H_{21})$ does not change under diagonal scalings, and it equals $e^{w}$ for $\exp W$. Thus $p^2/q^2=e^w$. Also $w=\ip{W}{E}$.
\step{1}{2} Under the Fisher--Rao flow, $\dot w=\Delta$.
\pf The straight-through flow is $\dot W=-C$, so $\dot w=-\ip{C}{E}=\Delta$.
\step{1}{3} Under the logit gradient flow, $\dot w=2\Delta pq$.
\pf By Lemma~\ref{lem:differential}(c), $\dot W=-\Pfr_H(H\odot C)=-\lambda E$, where
\[
\lambda=\frac{g_H(H\odot C,E)}{g_H(E,E)}=\frac{\ip{C}{E}}{2/p+2/q}=-\frac{\Delta}{2}\,pq .
\]
Thus $\dot w=\ip{\dot W}{E}=-4\lambda=2\Delta pq$.
\step{1}{4} The equations \eqref{eq:n2} for $q$ hold.
\pf $q=1/(1+e^{w/2})$, so $\dot q=-\frac12pq\,\dot w$. Insert \sref{1}{2} and \sref{1}{3}.
\step{1}{5} For $\Delta<0$, the logit escape satisfies $q(t)\le q_0/(1-|\Delta|q_0t)$.
\pf By \sref{1}{4}, $\dot q=|\Delta|q^2(1-q)^2\le|\Delta|q^2$, so $\frac{\dd}{\dd t}(-1/q)\le|\Delta|$. Integrate.
\qedstep{1}{6}
\pf Steps \sref{1}{1}--\sref{1}{5} give all claims. For $\Delta<0$ the Fisher--Rao equation $\dot q=\frac{|\Delta|}2q(1-q)$ is the logistic equation.
\end{steps}
\end{proof}

\begin{proof}[Proof of Corollary~\ref{cor:price}]
\begin{steps}
\step{1}{1} $\prod_{l<L}\sigma_n(H_l)\ge c_0^{\,n-1}$.
\pf By Lemma~\ref{lem:two-block}, $|\det(U\T H_lU)|=\prod_{k\ge2}\sigma_k(H_l)\le\sigma_n(H_l)$, because all singular values are in $[0,1]$. Also $|\det(U\T H_{L:0}U)|\ge\sigma_n(H_{L:0})^{n-1}\ge c_0^{\,n-1}$. The subspace $\one^\perp$ is invariant under every $H_l$, so the determinant on $\one^\perp$ is multiplicative: $\det(U\T H_{L:0}U)=\prod_l\det(U\T H_lU)$.
\step{1}{2} At least half of the layers satisfy $\vartheta_l:=1-\sigma_n(H_l)\le2(n-1)\log(1/c_0)/L$.
\pf Because $1-s\le\log(1/s)$ for $s\in(0,1]$, \sref{1}{1} gives $\sum_l\vartheta_l\le\sum_l\log(1/\sigma_n(H_l))\le(n-1)\log(1/c_0)$. Apply Markov's inequality.
\step{1}{3} These layers satisfy $\eps(H_l)\le\eps_\star$.
\pf For them, $4n\vartheta_l\le8n(n-1)\log(1/c_0)/L<1$. Proposition~\ref{prop:rigidity}(d) gives $\eps(H_l)\le4n\vartheta_l\le\eps_\star$.
\qedstep{1}{4}
\pf Step \sref{1}{3} gives the first claim. Solve $\eps_\star\ge\eps_0/(1+c\eps_0t)$ for $t$ in \eqref{eq:approach} for the first time bound. The bound \eqref{eq:fr-bounds} gives the second time bound, and Theorem~\ref{thm:fr-rate} shows that logarithmic time is possible.
\end{steps}
\end{proof}

\subsection{Proofs for Section~\ref{sec:sinkhorn}}

\begin{proof}[Proof of Proposition~\ref{prop:knight}]
\begin{steps}
\step{1}{1} $K_{ij}\nu^\star_j/(K\nu^\star)_i=H_{ij}$ and $\mu^\star_iK_{ij}/(K\T\mu^\star)_j=H_{ij}$.
\pf At the fixed point, $\mu^\star=\one\oslash(K\nu^\star)$ and $\nu^\star=\one\oslash(K\T\mu^\star)$, and $H_{ij}=\mu^\star_iK_{ij}\nu^\star_j$.
\step{1}{2} With $\nu=\nu^\star\odot(\one+v)$, $\mu(\nu)_i=\mu^\star_i\big(1-(Hv)_i+O(\|v\|^2)\big)$.
\pf By the first identity of \sref{1}{1}, $(K\nu)_i=(K\nu^\star)_i(1+(Hv)_i)$. Take the reciprocal.
\step{1}{3} $\Psi(\nu)=\nu^\star\odot(\one+H\T Hv+O(\|v\|^2))$, and the column sums of $\diag(\mu(\nu))K\diag(\nu)$ are $\one+(I-H\T H)v+O(\|v\|^2)$.
\pf By \sref{1}{2} and the second identity of \sref{1}{1}, $(K\T\mu(\nu))_j=(K\T\mu^\star)_j(1-(H\T Hv)_j+O(\|v\|^2))$. Take the reciprocal for $\Psi$. The column sums are $(K\T\mu(\nu))_j\nu_j=(1-(H\T Hv)_j)(1+v_j)+O(\|v\|^2)$.
\step{1}{4} $H\T H$ is symmetric and positive semidefinite, $H\T H\one=\one$, and on $\one^\perp$ its eigenvalues are $\sigma_k(H)^2\in[0,\sigma_2(H)^2]$, $k\ge2$, with $\sigma_2(H)<1$.
\pf Lemma~\ref{lem:two-block} and Proposition~\ref{prop:rigidity}(c).
\step{1}{5} $r_k\le\sigma_2(H)^{2k}\|Qv_0\|_2$.
\pf Write $v_0=\frac1n(\one\T v_0)\one+Qv_0$. By \sref{1}{4}, $(I-H\T H)\one=0$, so $(I-H\T H)(H\T H)^kv_0=(I-H\T H)(H\T H)^kQv_0$. On $\one^\perp$, $\|I-H\T H\|_2\le1$ and $\|(H\T H)^k\|_2=\sigma_2(H)^{2k}$. Finally $\|\cdot\|_\infty\le\|\cdot\|_2$.
\step{1}{6} If $Qv_0$ has a component $w\ne0$ in the eigenspace of $H\T H$ for $\sigma_2(H)^2$, then $r_k\ge(1-\sigma_2(H)^2)\,\sigma_2(H)^{2k}\|w\|_2/\sqrt n$.
\pf The component of $(I-H\T H)(H\T H)^kv_0$ in this eigenspace is $(1-\sigma_2^2)\sigma_2^{2k}w$. An orthogonal projection does not increase the norm, and $\|\cdot\|_\infty\ge\|\cdot\|_2/\sqrt n$.
\step{1}{7} The bounds \eqref{eq:sweeps-per-efold} hold.
\pf Write $\sigma_2=1-1/\tau$, and use $-\frac{s}{1-s}\le\log(1-s)\le-s$ for $s=1/\tau\in(0,1)$.
\qedstep{1}{8}
\pf By \sref{1}{3}--\sref{1}{7}.
\end{steps}
\end{proof}

\begin{proof}[Proof of Proposition~\ref{prop:drift}]
\begin{steps}
\step{1}{1} $H_{L:l}\one=\one$.
\pf A product of row-stochastic matrices is row-stochastic.
\step{1}{2} Every column sum of $H_{L:l}$ is in $\big[\prod_i(1-e_i),\prod_i(1+e_i)\big]$.
\pf Let $A,B\ge0$ have column sums in $[1-e_A,1+e_A]$ and $[1-e_B,1+e_B]$. Then, entrywise,
\[
(1-e_A)(1-e_B)\one\T\le(1-e_A)\one\T B\le\one\T AB=(\one\T A)B\le(1+e_A)\one\T B\le(1+e_A)(1+e_B)\one\T .
\]
Induction over the factors gives the claim.
\step{1}{3} $\prod_i(1+e_i)\le e^{\sum_ie_i}$ and $\prod_i(1-e_i)\ge e^{-2\sum_ie_i}$.
\pf $\log(1+s)\le s$, and $\log(1-s)\ge-2s$ for $s\in[0,\frac12]$.
\step{1}{4} $\|H_{L:l}\|_2\le e^{\frac12\sum_ie_i}$.
\pf By Lemma~\ref{aux:schur}, $\|H_i\|_2^2\le\|H_i\|_1\|H_i\|_\infty\le(1+e_i)\cdot1$. The operator norm is submultiplicative, and $1+e_i\le e^{e_i}$.
\qedstep{1}{5}
\pf By \sref{1}{1}--\sref{1}{4}.
\end{steps}
\end{proof}

\subsection{Proofs for Section~\ref{sec:isometric} and Appendix~\ref{app:extra}}

\begin{proof}[Proof of Proposition~\ref{prop:isometric} and of \eqref{eq:O1}]
\begin{steps}
\step{1}{1} \eqref{eq:O1} holds, and $\On$ is a compact Lie group isomorphic to $O(n-1)$.
\pf By Proposition~\ref{prop:rigidity} (steps \xsref{1}{2} and \xsref{1}{3} of its proof), $\On=\{J/n+UZU\T:Z\in O(n-1)\}$. The map $Z\mapsto J/n+UZU\T$ is an injective group homomorphism, because $(J/n+UZU\T)(J/n+UZ'U\T)=J/n+UZZ'U\T$, and it is smooth with a smooth inverse $A\mapsto U\T AU$.
\step{1}{2} $H_{L:l}\in\On$.
\pf $\On$ is a group by \sref{1}{1}.
\step{1}{3} Part (a) holds.
\pf Elements of $\On$ satisfy $A\one=\one$ and $A\T\one=A\T A\one=\one$. The proof of Theorem~\ref{thm:two-channel} uses only these two identities. By \sref{1}{2}, each $H_{L:i+1}$ is orthogonal, so it preserves Frobenius norms.
\step{1}{4} Part (b) holds.
\pf By Theorem~\ref{thm:two-channel}, the Jacobian is $x\mapsto H_{L:l}x$. All singular values of an orthogonal matrix, and of this map, are one.
\step{1}{5} $T_A\On=\{A\Omega:\ \Omega\T=-\Omega,\ \Omega\one=0\}$.
\pf Write $A=J/n+UZU\T$. Then $T_ZO(n-1)=\{Z\Omega':\ \Omega'\ \text{skew}\}$, so $T_A\On=\{UZ\Omega'U\T\}$. Because $JU=0$, $A(U\Omega'U\T)=UZ\Omega'U\T$. Also, $\{U\Omega'U\T:\ \Omega'\ \text{skew}\}$ is the set of skew $\Omega$ with $\Omega=Q\Omega Q$, that is, with $\Omega\one=0$.
\step{1}{6} The Riemannian gradient of $f$ for the Frobenius metric is $AQ\skewop(A\T\nabla f(A))Q$.
\pf Let $\mathfrak o:=\{\Omega\ \text{skew}:\ \Omega\one=0\}$. Because $A$ is orthogonal, $\ip{A\Omega_1}{A\Omega_2}=\ip{\Omega_1}{\Omega_2}$, so the orthogonal projection of $\nabla f(A)$ onto $T_A\On=A\mathfrak o$ is $A\,\Pi_{\mathfrak o}(A\T\nabla f(A))$. The maps $X\mapsto\skewop(X)$ and $X\mapsto QXQ$ are commuting orthogonal projections, because $\skewop(QXQ)=Q\skewop(X)Q$. Thus their product $X\mapsto Q\skewop(X)Q$ is the orthogonal projection onto the intersection of their ranges, which is $\mathfrak o$.
\step{1}{7} The Cayley map is a retraction on $\On$.
\pf By Lemma~\ref{aux:cayley}, $R_A(A\Omega)=A\,C(\Omega)$ is smooth and orthogonal, $R_A(0)=A$, and $\frac{\dd}{\dd t}R_A(tA\Omega)|_{t=0}=A\Omega$. Because $\Omega\one=0$, $(I+\frac12\Omega)\one=\one$ and $(I-\frac12\Omega)^{-1}\one=\one$. Thus $C(\Omega)\one=\one$, and $R_A(A\Omega)\one=A\one=\one$, so $R_A(A\Omega)\in\On$.
\qedstep{1}{8}
\pf By \sref{1}{1}--\sref{1}{7}; Proposition~\ref{prop:rigidity}(b) gives $\Sn\subset\On$.
\end{steps}
\end{proof}

\begin{proof}[Proof of Proposition~\ref{prop:rgd-conv}]
\begin{steps}
\step{1}{1} The pullbacks $f\circ R_A$ have Hessians that are bounded on a neighborhood of the origin, uniformly in $A\in\On$.
\pf $\On$ is compact, the Cayley retraction $R$ of Proposition~\ref{prop:isometric}(c) is smooth, and $f$ is twice continuously differentiable near $\On$.
\qedstep{1}{2}
\pf Let $\Lambda$ be the bound of \sref{1}{1}. The standard analysis of Riemannian gradient descent with a constant step at most $1/\Lambda$ gives $\min_{k<N}\|\gradFR_{\mathcal O}f(A^{(k)})\|_F^2\le2\Lambda\,(f(A^{(0)})-\min f)/N$ \citep[Chapter~4]{boumal2023introduction}.
\end{steps}
\end{proof}

%% file: app_discussion.tex
\section{Discussion}\label{app:discussion}

\subsection{What the theory says about mHC}

The doubly stochastic constraint gives one property, and it gives this property completely. The stream mean propagates through a residual network with unit gain, for every mixer and every depth. All other effects of the mixers stay in the difference channel, where the mixers contract. Thus the extra width of mHC is a short memory of $O(\tau)$ layers on top of a standard residual stream. The measurements of \citet{xie2026mhc,alimaskina2026stream,zhao2026howdoes,guo2026ohc} agree with this prediction.

The Sinkhorn-logit chart puts the collapse-free mixers at infinity, and logit gradient flow moves toward them only at rate $1/t$. The straight-through update is exact mirror descent, it moves at an exponential rate, and it needs no backward pass through Sinkhorn. Mixers with a long horizon are also the mixers that Sinkhorn projects most slowly. For a trained model, two cheap diagnostics follow: $\sigma_2(H_l)$ gives the horizon of layer $l$, and the final column-sum error bounds the drift of the backward gain (Proposition~\ref{prop:drift}). Table~\ref{tab:dynamics} compares the four ways to move a mixer that we study.

\begin{center}\small
\captionof{table}{Four ways to move a residual mixer, and what the theory says about each.}\label{tab:dynamics}
\begin{tabularx}{\linewidth}{@{}>{\raggedright\arraybackslash}p{0.19\linewidth}>{\raggedright\arraybackslash}p{0.16\linewidth}>{\raggedright\arraybackslash}X>{\raggedright\arraybackslash}p{0.18\linewidth}>{\raggedright\arraybackslash}p{0.10\linewidth}@{}}
\toprule
update rule & induced dynamics on $H$ & geometry & approach and escape of vertices & backward pass through $\SK$\\
\midrule
exact gradient of $f\circ\Phi$ (logit gradient descent) & $\dot H=-A_H^{2}\,\nabla f(H)$ & squared Fisher--Rao metric (Theorem~\ref{thm:induced}(a)) & polynomial, rate $1/t$ (Theorem~\ref{thm:viscosity}) & yes\\
straight-through, $\dot W=-\nabla f(H)$ & $\dot H=-A_H\nabla f(H)$ & Fisher--Rao metric; entropic mirror descent; Riemannian gradient descent with the Sinkhorn retraction (Theorem~\ref{thm:induced}(b,c)) & exponential (Theorem~\ref{thm:fr-rate}) & no\\
Adam on the logits & normalized per coordinate & sign-like near the vertices (Remark~\ref{rem:adam}) & exponential, empirically (Figure~\ref{fig:main}) & yes\\
Riemannian gradient descent on $\On$ & Cayley steps on $O(n-1)$ & Frobenius metric; no boundary (Section~\ref{sec:isometric}) & no vertices to approach & none\\
\bottomrule
\end{tabularx}
\end{center}

\subsection{Limits}\label{app:limits}

The theory does not say if a long horizon helps language modeling. It shows only that the Birkhoff design cannot have one at moderate budgets. Fast commitment to a vertex can also be premature \citep{liu2026premature}, and adaptive optimizers already reduce the viscosity (Remark~\ref{rem:adam}). Our results on input-dependent logits hold input by input (Remark~\ref{rem:dynamic}).

\subsection{Open problems}\label{app:open}

\begin{enumerate}[leftmargin=1.6em,itemsep=2pt]
\item \emph{Adaptive dynamics on the polytope.} Theorem~\ref{thm:viscosity} is about gradient flow. A rate theorem for Adam-type dynamics through the Sinkhorn chart would show how much viscosity remains in the regime in which mHC trains. The sign-descent analyses of adaptive methods are a possible start.
\item \emph{Escape rates.} Theorem~\ref{thm:fr-rate} shows that the Fisher--Rao flow \emph{can} approach a vertex at an exponential rate, for one loss. For $n>2$, two questions are open: a matching result for escape, and a characterization of the losses that attain the exponential rate. The second question is the analogue of the reduced-cost gap of linear programming.
\item \emph{Random products.} Theorem~\ref{thm:contraction} is a worst-case bound. The input-dependent mixers of mHC give a product of random doubly stochastic matrices. Its Lyapunov exponent on $\one^\perp$ is the natural definition of the effective horizon. The theory of Furstenberg and Kesten can give it in terms of the distribution of $\sigma_2$.
\item \emph{Exact parameterizations.} Each exact parameterization of $\Bn$ \citep{yang2026mhclite,zhou2026kromhc,lyubinin2026tbpmhc,dandachi2026gomhc,gu2026temper}, and the exact projection with implicit differentiation \citep{wang2026accelerating}, induces a metric on $\Bno$ through the analogue of Lemma~\ref{lem:differential}. A computation of these metrics near the vertices would extend Section~\ref{sec:viscosity} to the whole family.
\item \emph{Isometric designs.} By Proposition~\ref{prop:rigidity}, a smooth collapse-free design must allow negative entries, as the orthogonal and spectral-sphere designs do \citep{sengupta2026jpmhc,liu2026shc,guo2026ohc}. The subgroup $\On$ keeps the residual identity on the mean exactly. The full orthogonal group trades it for a coupling of mean and differences. It is an empirical question which inductive bias is better. The evidence of \citet{guo2026ohc} currently favors the full group.
\end{enumerate}